\documentclass{article} 
\usepackage[final]{colm2026_conference}

\usepackage{microtype}
\usepackage{hyperref}
\usepackage{url}
\usepackage{booktabs}
\usepackage{amsmath}
\usepackage{graphicx}
\usepackage{enumitem}
\usepackage{xspace}
\usepackage{multirow}
\usepackage{multicol}
\usepackage{cleveref}
\crefformat{section}{§#2#1#3}
\crefformat{subsection}{§#2#1#3}
\crefformat{subsubsection}{§#2#1#3}
\usepackage{caption} 
\usepackage{amsthm}
\newtheorem{proposition}{Proposition}
\usepackage{graphicx}
\usepackage{amssymb}
\usepackage{subcaption}
\usepackage{tcolorbox}
\usepackage{wrapfig}

\newcommand{\std}[1]{\scriptsize$\pm$#1}
\newcommand{\proposed}{VRF\xspace}

\usepackage{lineno}

\definecolor{darkblue}{rgb}{0, 0, 0.5}
\hypersetup{colorlinks=true, citecolor=darkblue, linkcolor=darkblue, urlcolor=darkblue}

\title{Uncertainty-Aware Variational Reward Factorization via \\ Probabilistic Preference Bases for LLM Personalization}

\author{Gyuseok Lee\textsuperscript{1} \quad Wonbin Kweon\textsuperscript{1} \quad Zhenrui Yue\textsuperscript{1} \\ \textbf{SeongKu Kang\textsuperscript{2} \quad Jiawei Han\textsuperscript{1} \quad Dong Wang\textsuperscript{1}} \\
\textsuperscript{1}University of Illinois Urbana-Champaign \quad \textsuperscript{2}Korea University \\
\texttt{\{gyuseok2,wonbin,zhenrui3,hanj,dwang24\}@illinois.edu}
 \quad \texttt{seongkukang@korea.ac.kr}
}

\begin{document}

\ifcolmsubmission
\linenumbers
\fi

\maketitle

\begin{abstract}
Reward factorization personalizes large language models (LLMs) by~decomposing rewards into shared basis functions and user-specific weights.
%
Yet, existing methods estimate user weights from scarce data in isolation and as deterministic points, leading to inaccurate and unreliable~inference.
%
We introduce Variational Reward Factorization (\proposed), an uncertainty-aware framework that represents each user's preferences as a variational distribution in a shared preference space.
%
\proposed infers user distributions via a variational encoder, derives weights through Wasserstein distance matching with shared probabilistic bases, and downweights uncertain estimates through a variance-attenuated loss.
%
~On three benchmarks, \proposed outperforms all baselines across seen and unseen users, few-shot scenarios, and varying uncertainty levels, with gains extending to downstream alignment.
Our code is available at \url{https://github.com/Gyu-Seok-Lee/VRF_COLM26}.
\end{abstract}

\section{Introduction}
\label{sec:introduction}


Aligning large language models (LLMs) with human preferences has attracted growing interest~\citep{ziegler2019fine, bai2022training, touvron2023llama, DPO}. 
One prominent line of research is reward modeling, which first learns preferences from pairwise feedback, then guides LLM alignment (e.g., RLHF)~\citep{christiano2017deep, stiennon2020learning, ouyang2022training}.
However, conventional approaches learn a single, shared reward model that inevitably collapses toward the average user, neglecting the diverse and often conflicting preferences~\citep{bakker2022fine, casper2023open, durmus2024towards}.

 
Beyond this one-size-fits-all assumption, personalized alignment of LLMs has been~explored~\citep{kirk2024benefits, sorensen2024position}. 
A leading approach is reward factorization~\citep{LoRe, PReF}, which decomposes a scalar reward as $r_u = \mathbf{w}_u^\top \boldsymbol{\phi}$, where $\mathbf{w}_u \in \Delta^{K-1}$ are user-specific weights and $\boldsymbol{\phi} \in \mathbb{R}^K$ are shared basis reward functions, with each $\phi_k$ capturing a distinct preference aspect. 
This factorization enables efficient personalization without training a per-user reward model~\citep{personalllm, cai2026one}.

However, existing reward factorization methods face two fundamental challenges in modeling user weights $\mathbf{w}_u$: (C1)~\textit{isolated preference modeling} and (C2)~\textit{deterministic point estimates}.\vspace{0.3em}
\textbf{(C1) Inaccurate estimation from isolated preference modeling:}
Inferring user preferences requires sufficient observations, 
yet per-user preference data is typically 
scarce\footnote{e.g., users in PRISM~\citep{prism} have only 5.7 conversations on average.}~\citep{christiano2017deep, ryan-etal-2025-synthesizeme}, leading to inaccurate estimation. 
While users often share common preference patterns that could compensate for limited data, existing methods estimate each user's weight in isolation, failing to leverage this commonality. 
This problem extends to unseen users, who need to estimate 
$\mathbf{w}_u$ from scratch without any learned structure.
We empirically confirm that performance generally degrades with less data in our experiments.\vspace{0.3em} 
%
%
%
%
\textbf{(C2) Unreliable inference from deterministic point estimates:}
User preference histories inherently carry uncertainty~\citep{casper2023open, DPL}, arising from both data scarcity and preference diversity.
Yet existing methods estimate the user weight as a deterministic point, which cannot express how reliably the estimate captures the user's underlying preferences.
As illustrated in Figure~\ref{fig:motivation},  users exhibit varying preference patterns: User A is consistent within a single style, while User B spreads across multiple styles. 
Intuitively, the estimate for User B should entail greater uncertainty than that for User A. 
However, without uncertainty quantification, deterministic point estimates implicitly assign uniform certainty to every user, leading to unreliable inference of user preferences.
\begin{figure}[t]
  \centering
  \makebox[\linewidth][c]{\includegraphics[width=1.02\linewidth]{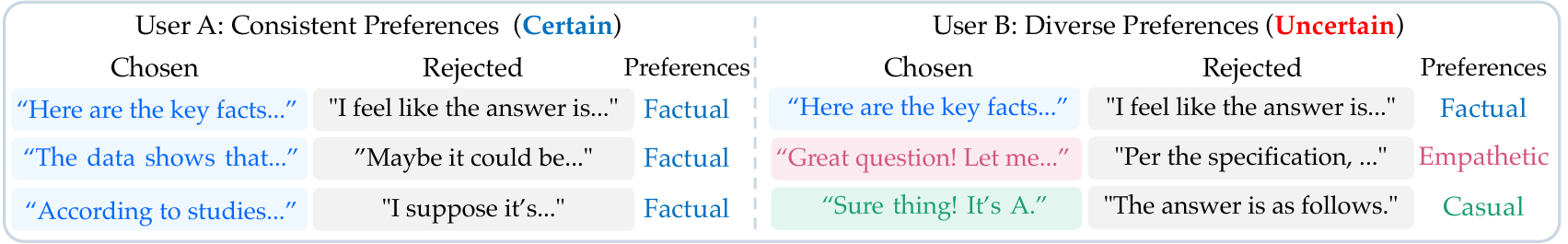}}
  \caption{Motivating example. Preferences can be consistent (certain) or diverse (uncertain).}
  \label{fig:motivation}
  \vspace{-0.4cm}
\end{figure}

In this paper, we propose \textbf{\underline{V}}ariational \textbf{\underline{R}}eward \textbf{\underline{F}}actorization (\textbf{\proposed}) for uncertainty-aware LLM personalization.
The core idea of \proposed is to represent each user not as a deterministic point but as a variational distribution in a shared preference space, from which user-specific weights are derived.
Specifically, \proposed introduces three key components: (1) a variational encoder that infers each user's preferences as a Gaussian distribution from a small reference set, 
(2) shared probabilistic preference bases with Wasserstein distance-based matching that jointly accounts for preference direction and confidence, and 
(3) a variance-attenuated preference loss that downweights uncertain reward estimates during training.

%
%
VRF systematically addresses both challenges. 
For \textbf{C1}, the shared preference bases encode population-level patterns, allowing data-scarce users to benefit from collective signals.
For \textbf{C2}, the probabilistic formulation of users and bases enables uncertainty-aware preference modeling, supported by the variance-attenuated loss.
At inference, the variational encoder with shared bases enables effective user weight estimation for unseen users in a single forward pass without any optimization.
Our core contributions are as follows:
\begin{itemize}[leftmargin=*]
    \item We identify two fundamental challenges of existing reward factorization approaches that remain underexplored in the literature, highlighting the need for systematic investigation.

    \item We propose VRF, an uncertainty-aware reward factorization framework that represents users and preference bases as probabilistic distributions over a shared preference space.
    

    \item We validate \proposed on three datasets, showing consistent improvements over~state-of-the-art baselines on seen and unseen users.~Further analyses highlight its robustness across varying data scarcity and uncertainty levels, with gains extending to downstream~alignment.
    
    
\end{itemize}
    
    

\section{Related Work}
\label{sec:related_work}
\paragraph{LLM Personalization.}
Personalizing LLMs to individual user preferences has been explored 
along several axes.
Prompt-based methods incorporate user histories~\citep{dai2023uncovering, li2023bookgpt} or retrieval-augmented context~\citep{richardson2023integrating, salemi2024lamp} into the input without additional training.
Other approaches personalize LLMs at the parametric level, including per-user 
adapters~\citep{tan2024personalized, tan2024democratizing}, user embeddings~\citep{doddapaneni2024user, liu2025llms+}, and steering vectors~\citep{cao2024personalized, zhang2025personalized}. 
While these methods implicitly encode user preferences at either the input or parametric level, our approach explicitly models them within the reward modeling framework.

\paragraph{Personalized Reward Modeling.}
Reward modeling directly learns human preferences from pairwise comparisons to guide LLM alignment (e.g., RLHF)~\citep{christiano2017deep, stiennon2020learning, 
ouyang2022training}.
However, conventional approaches learn a single, shared reward model that fails to capture diverse user preferences~\citep{bakker2022fine, casper2023open, durmus2024towards}.
To address this, reward factorization represents the personalized reward using shared basis functions and user-specific weights~\citep{LoRe, PAL, PReF}. 
Yet, existing methods face two key limitations in modeling user weights: (1) isolated per-user estimation that ignores shared preference patterns, and (2) deterministic point estimates that lack uncertainty quantification.
Although PAL~\citep{PAL} employs shared prototypes, they are designed for reward computation rather than user weight estimation, and its ideal~point formulation lacks uncertainty quantification.\footnote{Other approaches leverage cross-user signals but do not adopt reward factorization~\citep{li2024personalized, CoPL}, or share reward heads but route at the textual level without 
per-user inference~\citep{shen2025micro}. None of these methods address uncertainty in user weight estimation.
DPL~\citep{DPL} captures reward uncertainty but does not personalize to individual users.}
%
VPL~\citep{VPL} introduces variational inference for measuring uncertainty, but infers each user in isolation, leading to posterior collapse under sparse feedback~\citep{kim2026swap}.
These challenges motivate an uncertainty-aware reward factorization that leverages shared preference patterns beyond individual observations.

\vspace{-0.3cm}
\section{Preliminaries} 
\label{sec:preliminaries}
\subsection{Problem Formulation}
For each user $u \in \mathcal{U}$, we partition pairwise preference data into two disjoint sets: a reference set $\mathcal{C}_u$ for inferring user preferences and a target set $\mathcal{D}_u$ for computing the personalized reward. Each set contains triplets $(x, y^+, y^-)$, where $x \in \mathcal{X}$ is a query and $y^+, y^- \in \mathcal{Y}$ are chosen and rejected responses. Our goal is to learn a personalized reward function $r_u(x, y | \mathcal{C}_u)$ over $\mathcal{D}_u$.
The reward function is then used for personalized alignment of a policy model, either at training time like RLHF~\citep{christiano2017deep, ouyang2022training}~or at inference time like~best-of-N sampling~\citep{stiennon2020learning, touvron2023llama}.


\subsection{User-Agnostic Reward Modeling}
\label{subsec:standard_rm}
Early approaches to reward modeling~\citep{ziegler2019fine, stiennon2020learning, ouyang2022training} adopt the Bradley-Terry (BT) model \citep{bradley1952rank} to formulate the probability that response $y^+$ is preferred over $y^-$ given $x$ as:

\begin{equation}\label{eq:r_bt}
    P(y^+ \succ y^- \mid x) = \sigma\bigl(r_\theta(x, y^+) - r_\theta(x, y^-)\bigr),
\end{equation}

where $r_\theta(x, y) \in \mathbb{R}$ is a single reward function shared across all users and $\sigma(\cdot)$ is the sigmoid function.
$r_\theta$ is optimized by minimizing the negative log-likelihood as follows:

\begin{equation}
    \mathcal{L}_{\mathrm{BT}}(\theta) = -\mathbb{E}_{(x,y^+,y^-)\sim\mathcal{D}} \bigl[\log \sigma\bigl(r_\theta(x,y^+) - r_\theta(x,y^-)\bigr)\bigr],
\end{equation}

where $\mathcal{D} = \bigcup_{u \in \mathcal{U}} \mathcal{D}_u$.
Since $r_\theta$ is optimized over the aggregated data across users, it collapses heterogeneous preferences into a generic reward. For 
instance, given a prompt $x$, if User~A prefers a formal 
response $y_1$ over a casual one $y_2$, while User~B prefers the opposite, $r_\theta$ averages these conflicting signals ($r_\theta(x, y_1) \approx r_\theta(x, y_2)$), thereby failing to satisfy either user.






\subsection{Reward Factorization}
\label{subsec:reward_factorization}
To capture individual user preferences, recent work~\citep{LoRe, PReF} factorizes the single scalar reward as follows: 

\begin{equation}\label{eq:r_u(xy)}
    r_u(x, y) = \mathbf{w}_u^\top \boldsymbol{\phi}(x, y),
\end{equation}

where $\boldsymbol{\phi}(x, y) = [\phi_1(x, y), \dots, \phi_K(x, y)]^\top \in \mathbb{R}^K$ denotes $K$ shared basis reward functions parameterized by a neural network, with each $\phi_k$ scoring a distinct preference style (e.g., conciseness or formality).
$\mathbf{w}_u \in 
\Delta^{K-1}$ is a user weight that personalizes the reward by determining the relative importance among the bases. 
The training objective is as follows:

\begin{equation}\label{eq:L_RF}
    \mathcal{L}_{RF}(\mathbf{w}, \boldsymbol{\phi}) = - \mathbb{E}_{u \sim \mathcal{U}} \mathbb{E}_{(x, y^+, y^-) \sim \mathcal{D}_u} \left[ \log \sigma \big( r_u(x, y^+) - r_u(x, y^-) \big) \right].
\end{equation}

However, existing reward factorization methods estimate $\mathbf{w}_u$ in isolation~without leveraging shared preference patterns (C1), and represent it as a deterministic point without uncertainty quantification (C2), resulting in inaccurate and unreliable~personalization.
%


\section{Variational Reward Factorization} 
\label{sec:method}

\begin{figure}[t]
  \centering
  \makebox[\linewidth][c]{\includegraphics[width=1.02\linewidth]{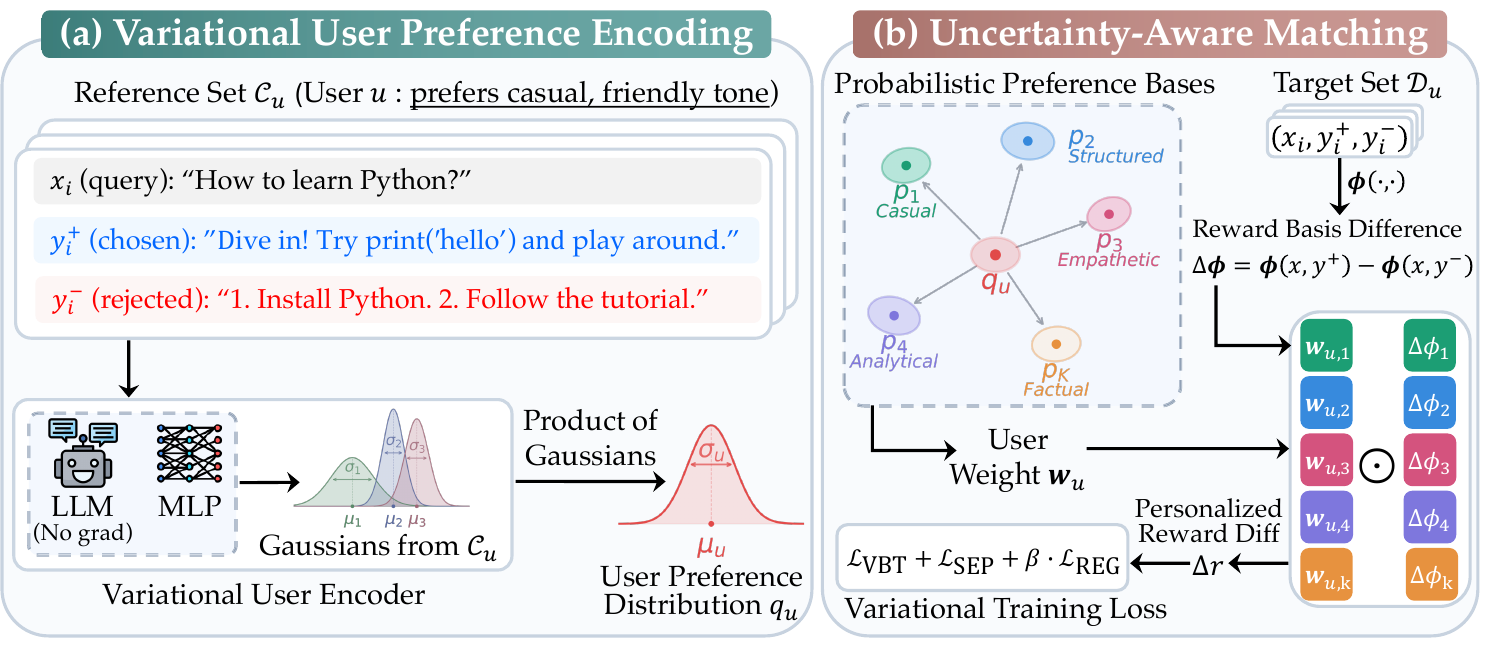}}
  \caption{Overall framework of \proposed. 
  Best viewed in color.}
  \label{fig:method}
  \vspace{-0.3cm}
\end{figure}


To address these challenges, we propose \textbf{Variational Reward Factorization} (\textbf{\proposed}), which represents each user's preferences not as a deterministic point, but as a variational distribution in a shared preference space, from which the user weight is derived. We first infer the user preference distribution (\cref{subsec:var_user_encoder}), then derive the user weight in probabilistic preference space (\cref{subsec:user_preference_modeling}). We further propose uncertainty-aware training objectives (\cref{subsec:training_objectives}). Finally, we present few-shot inference for unseen users (\cref{subsub:inference}).
Figure~\ref{fig:method} shows an overview of \proposed.


\subsection{Variational User Preference Encoding}
\label{subsec:var_user_encoder}
We propose a variational user encoder that infers each user's preference distribution $q_u = \mathcal{N}(\boldsymbol{\mu}_u, \mathrm{diag}(\boldsymbol{\sigma}_u^2))$ from a small reference set $\mathcal{C}_u$.
Here, $\boldsymbol{\mu}_u$ encodes the preference location and $\boldsymbol{\sigma}_u$ captures its uncertainty, which is inherently absent in deterministic point estimates.

\paragraph{Variational Pairwise Encoding.}
Given the reference set $\mathcal{C}_u$ for user $u$, we form a pairwise representation $\mathbf{v}_i \in \mathbb{R}^{2H}$ from each  triplet\footnote{$ \mathbf{v}_i = [\frac{\mathbf{h}_i^+ + 
\mathbf{h}_i^-}{2} \| \ \mathbf{h}_i^+ - 
\mathbf{h}_i^-] \in \mathbb{R}^{2H} $, where $\mathbf{h}_i^+ = \mathrm{LLM}(x_i \| y_i^+) \in \mathbb{R}^{H}$ and $\mathbf{h}_i^- = \mathrm{LLM}(x_i \| y_i^-) \in \mathbb{R}^{H}$ are last-token hidden states extracted under no gradient and $\|$ denotes concatenation. The mean and difference capture shared content and preference direction, respectively.} $(x_i, y_i^+, y_i^-)$.
Then, a two-layer MLP maps $\mathbf{v}_i$ to Gaussian parameters $\boldsymbol{\mu}_i, \boldsymbol{\sigma}_i \in \mathbb{R}^D$~\citep{DBLP:journals/corr/KingmaW13}, resulting in $|\mathcal{C}_u|$ Gaussians (i.e., $\{\mathcal{N}(\boldsymbol{\mu}_i, \mathrm{diag}(\boldsymbol{\sigma}_i^2))\}_{i=1}^{|\mathcal{C}_u|}$).

\paragraph{User Preference Aggregation.}
To aggregate $|\mathcal{C}_u|$ Gaussians into a single user distribution $q_u = \mathcal{N}(\boldsymbol{\mu}_u, \mathrm{diag}(\boldsymbol{\sigma}_u^2))$, we adopt the product of Gaussians~\citep{bishop2006pattern}, which yields 
a closed-form aggregation:
\begin{equation}\label{eq:mu_u and simga_u}
    \boldsymbol{\mu}_u = \frac{\sum_{i=1}^{M_u} \boldsymbol{\mu}_i \odot \boldsymbol{\sigma}_i^{-2}}{\sum_{i=1}^{M_u} \boldsymbol{\sigma}_i^{-2}}, \quad \boldsymbol{\sigma}_u^2 = \frac{1}{\sum_{i=1}^{M_u} \boldsymbol{\sigma}_i^{-2}}, 
\end{equation}
where $\odot$ is element-wise multiplication.
Unlike simple averaging, the product of Gaussians enables 
more reliable estimation, weighting each observation by 
its precision $\boldsymbol{\sigma}_i^{-2}$. 
Thus, $\boldsymbol{\mu}_u$ reflects more confident observations while $\boldsymbol{\sigma}_u$ decreases as more observations are available.


\subsection{Uncertainty-Aware Preference Matching}
\label{subsec:user_preference_modeling}
Given the user distribution $q_u = \mathcal{N}(\boldsymbol{\mu}_u, \mathrm{diag}(\boldsymbol{\sigma}_u^2))$, we derive the user weight $\mathbf{w}_u \in \Delta^{K-1}$ over $K$ shared probabilistic preference bases (i.e., Gaussians).
By representing both users and bases as distributions, we enable uncertainty-aware preference matching between them.
%
%
\paragraph{Probabilistic Preference Bases.}
\label{subsub:prob_preference_bases}
We maintain $K$ shared preference bases 
$\mathcal{P} = \{p_k\}_{k=1}^K$, where 
$p_k = \mathcal{N}(\boldsymbol{\mu}_k, 
\mathrm{diag}(\boldsymbol{\sigma}_k^2))$ with 
$\boldsymbol{\mu}_k, \boldsymbol{\sigma}_k \in 
\mathbb{R}^D$ as learnable parameters. Each $p_k$ 
represents a distinct preference style (e.g., conciseness, 
formality, or empathy), corresponding to the $k$-th reward function $\phi_k(x, y)$. 
The bases $\mathcal{P}$ are randomly initialized and learned to capture domain-specific preference patterns that minimize the overall training objective.

The preference bases $\mathcal{P}$ offer three key benefits: (1) sharing bases 
across users encodes population-level patterns, allowing 
data-sparse and unseen users to benefit from collective signals (C1); 
(2) modeling bases as Gaussians enables uncertainty-aware 
matching with $q_u$, naturally incorporating preference 
uncertainty into $\mathbf{w}_u$ (C2); and (3) $\mathcal{P}$ is 
parameter-efficient, as the number of bases remains fixed 
regardless of the user population size.


\paragraph{User Weight Derivation via Wasserstein Distance.}
We derive the user weight $\mathbf{w}_u$ via the squared 2-Wasserstein distance\footnote{ 
Unlike KL divergence, $\mathcal{W}_2$ is symmetric and considers the geometric structure of the latent space.} $\mathcal{W}_2(\cdot,\cdot)^2$ between the user distribution $q_u$ and each preference basis $p_k$. For diagonal Gaussians, this admits a closed form $\mathcal{W}_2(q_u, p_k)^2 = \frac{1}{D} ( \|\boldsymbol{\mu}_u - \boldsymbol{\mu}_k\|_2^2 + \|\boldsymbol{\sigma}_u - \boldsymbol{\sigma}_k\|_2^2 )$~\citep{villani2008optimal}, which naturally accounts for both preference location and uncertainty. 
We then derive $\mathbf{w}_u$ from these distances via 
softmax:
\begin{equation}\label{eq:w_uk}
  \mathbf{w}_u = [w_{u,1}, \dots, w_{u,K}]^\top \in 
  \Delta^{K-1}, \quad
  w_{u,k} = \frac{\exp(-\mathcal{W}_2(q_u, p_k)^2 / 
  \tau_d)}{\sum_{j=1}^K \exp(-\mathcal{W}_2(q_u, p_j)^2 / 
  \tau_d)},
\end{equation}
where $\tau_d$ is a temperature hyperparameter. The 
resulting $\mathbf{w}_u$ encodes an~uncertainty-aware soft assignment over preference bases, yielding the personalized reward in Eq.~\eqref{eq:r_u(xy)}.
\subsection{Training Objectives}
\label{subsec:training_objectives}
%
\paragraph{Variance-Attenuated Bradley-Terry Loss.}
Under the original Bradley-Terry loss $\mathcal{L}_{\mathrm{BT}}$, the pairwise preference probability is computed as $P(y^+ \succ y^- \mid x) = \sigma(\Delta r_\mathrm{BT})$, where $\Delta r_\mathrm{BT} = r_\theta(x, y^+) - r_\theta(x, y^-)$ is a scalar reward difference.
However, this formulation does not account for uncertainty in reward estimation.
To address this, we propose a variance-attenuated preference loss that replaces the scalar $\Delta r_\mathrm{BT}$ with a Gaussian $\Delta r_u \sim \mathcal{N}(\mu_\Delta, \sigma^2_\Delta)$.

%
ddLet $\Delta\boldsymbol{\phi} = [\Delta\phi_1, \dots, \Delta\phi_K]^\top \in \mathbb{R}^K$ denote the reward basis differences, where $\Delta\phi_k = \phi_k(x, y^+) - \phi_k(x, y^-)$. 
Given the user weight $\mathbf{w}_u$, the Gaussian parameters are computed as:
\begin{equation}
    \mu_\Delta = \sum_{k=1}^K w_{u,k} \Delta\phi_k,
    \quad
    \sigma_\Delta^2 = \sum_{k=1}^K w_{u,k} (\Delta\phi_k - \mu_\Delta)^2,
\end{equation}
where $\mu_\Delta$ and $\sigma_\Delta^2$ are the mean and variance\footnote{$\sigma_u$ and $\sigma_\Delta$
capture uncertainty at different levels: $\sigma_u$ over the user preference
distribution $q_u$, and $\sigma_\Delta$ over the reward estimate for a specific pair given $\mathbf{w}_u$.} of the reward basis differences weighted by $\mathbf{w}_{u}$.
We approximate $\Delta r_u \sim \mathcal{N}(\mu_\Delta, \sigma^2_\Delta)$ via moment matching~\citep{minka2001expectation}.
The pairwise preference probability of user $u$ then becomes $P(y^+ \succ y^- \mid x, u) = \mathbb{E}_{\Delta r_u}[\sigma(\Delta r_u)]$.
As this expectation is analytically intractable\footnote{$\mathbb{E}_{\Delta r_u}[\sigma(\Delta r_u)] = \int_{-\infty}^{\infty} \sigma(\Delta r_u) \mathcal{N}(\Delta r_u \mid \mu_\Delta, \sigma_\Delta^2) \, d(\Delta r_u)$, which has no analytic solution.}, we apply MacKay's Approximation~\citep{mackay1992evidence} to derive a closed-form variance-attenuated Bradley-Terry loss:

\begin{equation}\label{eq:l_vbt}
    \mathcal{L}_{\mathrm{VBT}}(\mathbf{w}, \boldsymbol{\phi}, \mathcal{P}) = - \mathbb{E}_{u \sim \mathcal{U}} \, \mathbb{E}_{(x, y^+, y^-) \sim \mathcal{D}_u} \!\left[ \log \sigma\!\left( \tfrac{\mu_\Delta}{\sqrt{1 + \pi \sigma_\Delta^2 / 8}} \right) \right].
\end{equation}

$\mathcal{L}_{\mathrm{VBT}}$ adaptively scales the gradient via $\sigma^2_\Delta$, which depends on the concentration of $\mathbf{w}_u$: (1) when $\mathbf{w}_u$ is concentrated (e.g., one-hot), $\sigma^2_\Delta = 0$ and $\mathcal{L}_{\mathrm{VBT}}$ recovers $\mathcal{L}_{\mathrm{BT}}$;
(2) as $\mathbf{w}_u$ becomes more diffuse, $\sigma^2_\Delta$ grows and attenuates the gradient (see Appendix~\ref{app:concavity} and~\ref{app:proof_vbt}).
Consequently, $\mathcal{L}_{\mathrm{VBT}}$ enables uncertainty-aware reward estimation.

\paragraph{Variational Training Objective.}
Following variational inference~\citep{DBLP:journals/corr/KingmaW13}, we derive the final training objective as:

\begin{equation}
    \mathcal{L}(\mathbf{w}, \boldsymbol{\phi}, \mathcal{P}) = \mathcal{L}_{\mathrm{VBT}} + \mathcal{L}_{\mathrm{SEP}} + \beta \cdot \mathcal{L}_{\mathrm{REG}}.
    \label{eq:total}
\end{equation}

$\mathcal{L}_{\mathrm{SEP}} = \sum_{1 \le i < j \le K} \max\!\left(0,\, \tau_{\mathrm{m}} - \mathcal{W}_2(p_i, p_j)^2\right)$ promotes diversity among the $K$ preference 
bases with margin $\tau_m$.
%
$\mathcal{L}_{\mathrm{REG}} = \mathbb{E}_{u \sim \mathcal{U}}[D_{\mathrm{KL}}(q_u \parallel p(\mathbf{z}))]$~regularizes $q_u$ to lie within the learned preference space using a mixture-of-Gaussians prior\footnote{Since $D_{\mathrm{KL}}$ against a mixture-of-Gaussians prior lacks a closed form, we estimate it via Monte Carlo sampling with the reparameterization trick (see Appendix~\ref{app:mc}).} $p(\mathbf{z}) = \frac{1}{K}\sum_{k=1}^K p_k$ instead of a standard $\mathcal{N}(\mathbf{0}, \mathbf{I})$.
$\beta$ controls the regularization strength.

\subsection{Few-Shot Inference for Unseen Users}
\label{subsub:inference}
At inference time, the basis reward functions $\boldsymbol{\phi}$ are known, but the user weight $\mathbf{w}_{u'}$ for an unseen user $u'$ is derived on-the-fly.\footnote{For seen users, the learned weights $\mathbf{w}_u$ are cached during training and applied directly.}
Given a small reference set $\mathcal{C}_{u'}$, the variational encoder infers the user distribution $q_{u'}$ (\cref{subsec:var_user_encoder}), from which $\mathbf{w}_{u'}$ is derived (\cref{subsec:user_preference_modeling}). 
The personalized reward is then obtained via Eq.~\eqref{eq:r_u(xy)}.
Unlike methods requiring gradient-based optimization~\citep{LoRe,PAL,PReF}, \proposed estimates $\mathbf{w}_{u'}$ via a single forward pass without additional optimization.

\section{Experiments}
\label{sec:experiments}




\subsection{Experimental Setup}
\label{sec:setup}
\paragraph{Datasets.}
We evaluate on three benchmarks: \textbf{PersonalLLM}~\citep{personalllm}, where $\alpha \in \{0.001, 0.01, 0.1\}$ controls preference diversity; \textbf{TL;DR}~\citep{stiennon2020learning}, where $n \in \{100, 150\}$ denotes the number of training triplets per seen user; and \textbf{PRISM}~\citep{prism}.
Detailed statistics are provided in Appendix~\ref{appendix:datasets}.

\paragraph{Evaluation Metric.}
We report pairwise preference accuracy as the primary metric: $\text{Acc} = \frac{1}{|\mathcal{D}_{\text{test}}|} \sum_{(x, y^+, y^-, u) \in \mathcal{D}_{\text{test}}} \mathbb{1}[r_u(x, y^+) > r_u(x, y^-)]$.
Results are averaged over 5 random seeds, controlling model initialization and data splits.
Performance is evaluated across three distinct splits: Seen, Unseen, and Overall (the macro-average of Seen and Unseen).

\paragraph{Baselines.}
We compare \proposed with several baselines: the pretrained Ref~\citep{qwen3}, non-personalized BT~\citep{stiennon2020learning, ouyang2022training}, and personalized methods including VPL~\citep{VPL}, PAL~\citep{PAL}, LoRe~\citep{LoRe}, and PReF~\citep{PReF}.
Detailed explanations of each baseline are in Appendix~\ref{appendix:baselines}.

\paragraph{Implementation Details.}
We employ Qwen3-0.6B \citep{qwen3} as the backbone for the basis reward functions\footnote{$\boldsymbol{\phi}$ comprises the backbone and a two-layer 
MLP ($H \to H/4 \to K$) on the last-token hidden state.} $\boldsymbol{\phi}$, fully fine-tuned given its compact size.
For \proposed and VPL, the backbone is additionally used as the variational encoder without gradient updates.\footnote{As the shared backbone is updated through $\boldsymbol{\phi}$, the encoder reflects preference learning without additional backward passes.} 
All baselines are tuned following their original protocols.
Detailed hyperparameter configurations are provided in Appendix~\ref{appendix:implementation_details}.
Backbone scale from 0.6B to 8B is analyzed in Appendix~\ref{app:backbone_scaling}.

\subsection{Performance Comparison}
\label{subsec:preformance comparison}
\begin{table*}[t]
\centering
\caption{Overall performance comparison on three datasets.}
\label{tab:main_combined}
\small
\setlength{\tabcolsep}{4.5pt}
\begin{tabular}{l|ccc|ccc|ccc}
\toprule
 & \multicolumn{9}{c}{\textbf{PersonalLLM}} \\
\textbf{Method} & \multicolumn{3}{c|}{Very Diverse ($\alpha{=}0.001$)} & \multicolumn{3}{c|}{Moderately Diverse ($\alpha{=}0.01$)} & \multicolumn{3}{c}{Near Uniform ($\alpha{=}0.1$)} \\
& Seen & Unseen & Overall & Seen & Unseen & Overall & Seen & Unseen & Overall \\
\midrule\midrule
Ref & 58.6\std{0.2} & 58.9\std{0.9} & 58.8\std{0.4} & 59.1\std{0.3} & 59.0\std{0.6} & 59.1\std{0.4} & 59.9\std{0.2} & 60.0\std{0.2} & 60.0\std{0.2} \\
BT  & 90.3\std{0.3} & \underline{90.5\std{0.4}} & \underline{90.4\std{0.2}} & 91.1\std{0.4} & \underline{91.0\std{0.3}} & \underline{91.1\std{0.2}} & \underline{94.8\std{0.2}} & \underline{94.6\std{0.4}} & \underline{94.7\std{0.2}} \\
VPL & 86.4\std{0.4} & 86.9\std{0.5} & 86.6\std{0.2} & 87.3\std{0.5} & 87.4\std{0.3} & 87.3\std{0.2} & 91.7\std{0.2} & 91.5\std{0.5} & 91.6\std{0.3} \\
PAL & 88.8\std{0.5} & 89.3\std{0.1} & 89.0\std{0.3} & 89.3\std{0.5} & 89.2\std{0.3} & 89.2\std{0.2} & 93.0\std{0.2} & 92.8\std{0.4} & 92.9\std{0.3} \\
LoRe & 89.4\std{1.2} & 88.6\std{0.3} & 89.0\std{0.5} & 89.6\std{0.9} & 88.7\std{0.2} & 89.2\std{0.6} & 92.7\std{0.2} & 92.5\std{0.3} & 92.6\std{0.2} \\
PReF & \underline{94.4\std{0.2}} & 81.1\std{0.1} & 87.7\std{0.1} & \underline{94.0\std{0.3}} & 82.4\std{0.5} & 88.2\std{0.2} & 94.3\std{0.1} & 90.4\std{0.4} & 92.3\std{0.2} \\
\midrule
\textbf{VRF} & \textbf{95.6\std{0.4}} & \textbf{95.1\std{0.3}} & \textbf{95.3\std{0.3}} & \textbf{95.5\std{0.1}} & \textbf{94.7\std{0.2}} & \textbf{95.1\std{0.1}} & \textbf{96.8\std{0.0}} & \textbf{96.4\std{0.2}} & \textbf{96.6\std{0.1}} \\
\midrule\midrule
\multirow{2}{*}{\textbf{Method}} & \multicolumn{3}{c|}{\textbf{TL;DR ($n{=}100$)}} & \multicolumn{3}{c|}{\textbf{TL;DR ($n{=}150$)}} & \multicolumn{3}{c}{\textbf{PRISM}} \\
& Seen & Unseen & Overall & Seen & Unseen & Overall & Seen & Unseen & Overall \\
\midrule\midrule
Ref & 51.9\std{0.2} & 51.8\std{0.7} & 51.8\std{0.3} & 51.9\std{0.2} & 51.8\std{0.7} & 51.8\std{0.3} & 54.7\std{0.7} & 54.9\std{0.3} & 54.8\std{0.5} \\
BT  & 59.9\std{0.7} & 59.0\std{0.7} & 59.5\std{0.7} & 61.5\std{0.6} & 60.4\std{0.6} & 60.9\std{0.5} & 60.0\std{0.7} & 60.2\std{0.9} & 60.1\std{0.5} \\
VPL & 58.6\std{1.0} & 57.5\std{0.4} & 58.1\std{0.7} & 58.7\std{1.2} & 56.7\std{0.3} & 57.7\std{0.5} & 59.6\std{0.7} & 58.8\std{1.0} & 59.2\std{0.7} \\
PAL & 60.4\std{0.7} & 59.1\std{0.3} & 59.8\std{0.4} & 61.7\std{0.5} & 60.6\std{0.4} & 61.2\std{0.1} & \underline{60.8\std{0.4}} & \underline{60.8\std{0.9}} & \underline{60.8\std{0.5}} \\
LoRe & \underline{61.6\std{0.5}} & \underline{61.0\std{0.6}} & \underline{61.3\std{0.6}} & \underline{62.3\std{0.8}} & \underline{61.0\std{0.7}} & \underline{61.7\std{0.6}} & 60.4\std{1.0} & 59.7\std{0.9} & 60.0\std{0.9} \\
PReF & 55.7\std{2.5} & 56.7\std{2.2} & 56.2\std{1.9} & 59.9\std{1.3} & 59.2\std{1.8} & 59.5\std{1.5} & 53.1\std{1.6} & 51.4\std{1.2} & 52.3\std{1.3} \\
\midrule
\textbf{VRF} & \textbf{62.3\std{1.2}} & \textbf{61.5\std{0.9}} & \textbf{61.9\std{1.0}} & \textbf{63.3\std{1.1}} & \textbf{62.4\std{1.1}} & \textbf{62.9\std{1.1}} & \textbf{62.1\std{0.6}} & \textbf{61.6\std{0.5}} & \textbf{61.8\std{0.3}} \\
\bottomrule
\end{tabular}
\end{table*}
\paragraph{Main Results.}
Table~\ref{tab:main_combined} shows that \proposed consistently achieves the best performance across all datasets and splits, indicating its effectiveness for personalized reward modeling across synthetic and real-world settings.
The gains over BT highlight the importance of user-specific modeling, while improvements over VPL and PAL suggest that combining variational user representations with shared preference bases provides complementary benefits.
Unlike LoRe and PReF, which rely on deterministic user representations, \proposed represents both users and preference bases as probabilistic distributions, enabling uncertainty-aware preference modeling.
This allows the model to quantify preference uncertainty and attenuate ambiguous preference signals during training, leading to more robust preference learning.
%
Notably, on PersonalLLM, the performance gap widens as user preferences become more diverse ($\alpha{=}0.001$), suggesting that \proposed more effectively captures heterogeneous user preferences.
\paragraph{In-Depth Comparison.}
Figure~\ref{fig:combined} examines unseen users in PRISM across four dimensions: (a) few-shot adaptation, (b) uncertainty robustness, (c) inference-time alignment, and (d) adaptation efficiency.
For (a) and (b), PReF is excluded due to its unstable performance.
Note that (a) evaluates reward model accuracy by varying $|\mathcal{C}_u|$, while (c) measures downstream alignment quality via best-of-$N$ sampling, grouping users by their original reference set size.
Additional few-shot adaptation and uncertainty robustness results on PersonalLLM and TL;DR are provided in Appendix~\ref{appendix:uncertain_fewshot}.
%
\begin{enumerate}[label=(\alph*), leftmargin=*, nosep]
    \item \textbf{Few-Shot Adaptation.}
    Figure~\ref{fig:fewshot} reports accuracy by varying $|\mathcal{C}_u| \in \{1, 3, 5, 7, 9\}$.
    \proposed achieves the highest accuracy across all values of $|\mathcal{C}_u|$.
    Notably, the improvement over the best competitor (PAL) is largest at $|\mathcal{C}_u| =1$ ({+}1.9\%), indicating that \proposed effectively infers user preferences even from minimal user history.
    This result suggests that shared preference bases enable data-scarce users to leverage collective knowledge from the user population. The variational encoder, regularized by the mixture-of-Gaussians prior, further maps users within this learned preference space.
    Together, these components address the isolated estimation problem (C1).

    \item \textbf{Uncertainty Robustness.}
    Figure~\ref{fig:uncert} categorizes unseen users into three bins (low, mid, high) based on their preference uncertainty
    $\sigma_u$.\footnote{To obtain a single uncertainty score per user, we average $\boldsymbol{\sigma}_u \in \mathbb{R}^D$ across dimensions.}
    \proposed consistently outperforms all baselines across all uncertainty levels.
    This robustness stems from our uncertainty-aware distributional matching  between users and preference bases and variance-attenuated training, jointly mitigating the unreliable inference from deterministic point estimates~(C2).
    

    \item \textbf{Inference-Time Alignment.}
    We evaluate inference-time alignment via best-of-$N$ sampling  to assess whether improved reward estimation translates to better response selection (see Appendix~\ref{appendix:inf_setup} for detailed setup).
    Figure~\ref{fig:winrate} shows that \proposed consistently achieves the highest $\Delta$ Win Rate across all user groups.
    Even with only 3--5 reference pairs, \proposed outperforms all baselines.
    These results suggest that the accuracy gains observed in (a) and (b) translate to tangible improvements in downstream alignment.
    
    \item \textbf{Adaptation Efficiency.}
    Figure~\ref{fig:adpat_time} reports per-user adaptation latency.
    'Backbone' denotes the time for LLM forward pass over the reference set, and 'Adapt' measures the time for estimating user weights.
    \proposed and VPL compute user weights via a single forward pass (0.6ms and 0.2ms, respectively), whereas gradient-based methods (PAL, LoRe, PReF) require 83--114ms.
    This advantage becomes critical at scale; for 1M users, \proposed's adaptation completes in under 10 minutes, whereas PReF requires over 30 hours.
    \proposed thus enables efficient personalization with negligible adaptation overhead, making it practical for large-scale deployment.
\end{enumerate}
Taken together, \proposed effectively addresses both the isolated estimation (C1) and unreliable inference (C2) challenges across few-shot scenarios and diverse uncertainty levels.
These gains translate to downstream personalized alignment while maintaining practical efficiency.

\begin{figure*}[t]
    \centering
    \begin{minipage}[b]{0.48\linewidth}
        \centering
        \includegraphics[width=\linewidth]{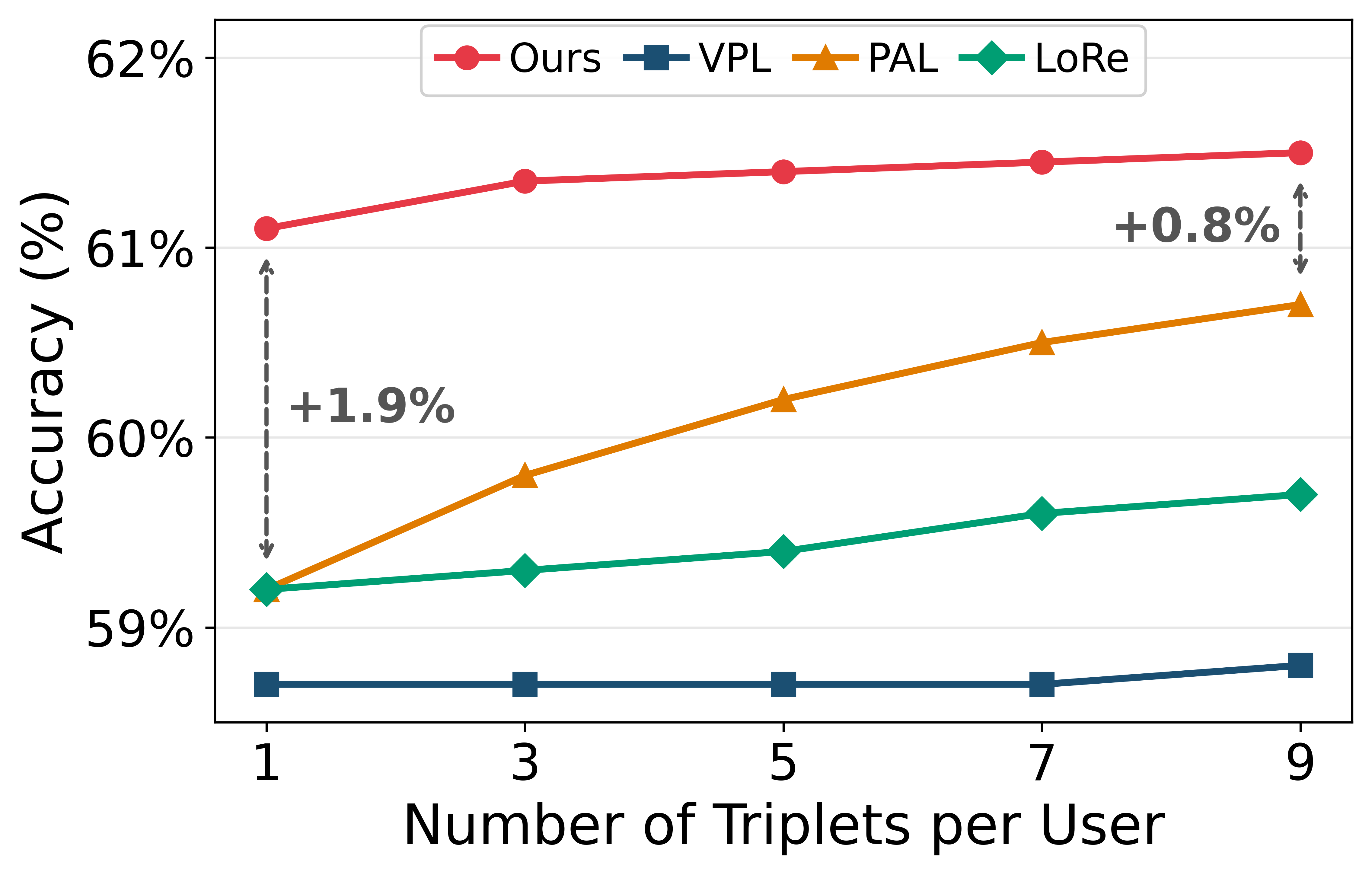}
        \subcaption{Few-shot adaptation (Varying $|\mathcal{C}_u|$)}
        \label{fig:fewshot}
    \end{minipage}
    \hfill
    \begin{minipage}[b]{0.48\linewidth}
        \centering
        \includegraphics[width=\linewidth]{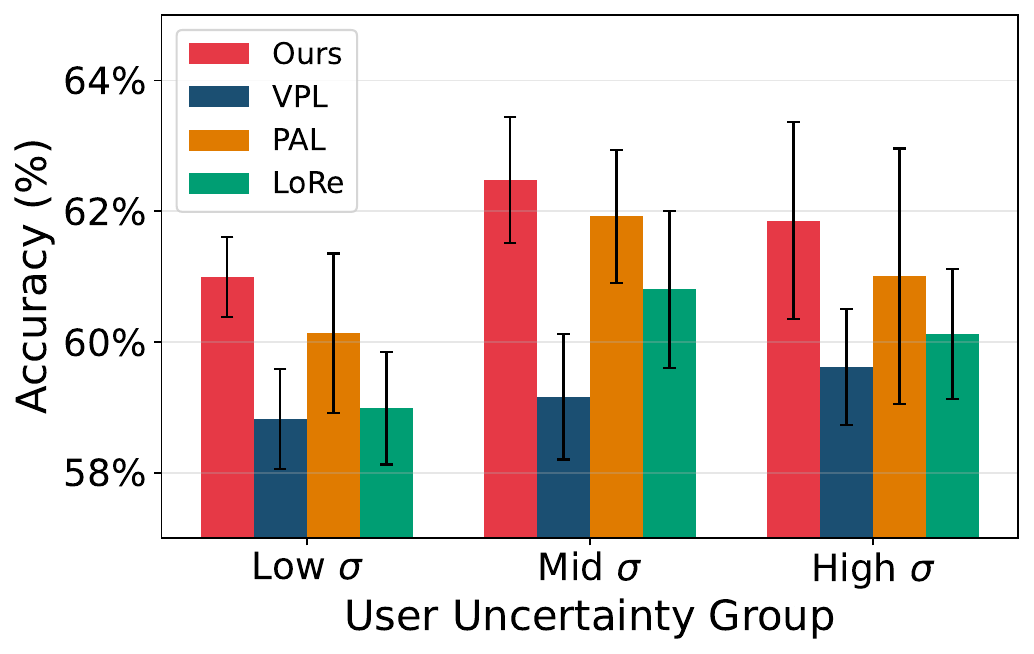}
        \subcaption{Uncertainty robustness}
        \label{fig:uncert}
    \end{minipage}

    \vspace{5pt}

    \begin{minipage}[b]{0.48\linewidth}
        \centering
        \includegraphics[width=\linewidth]{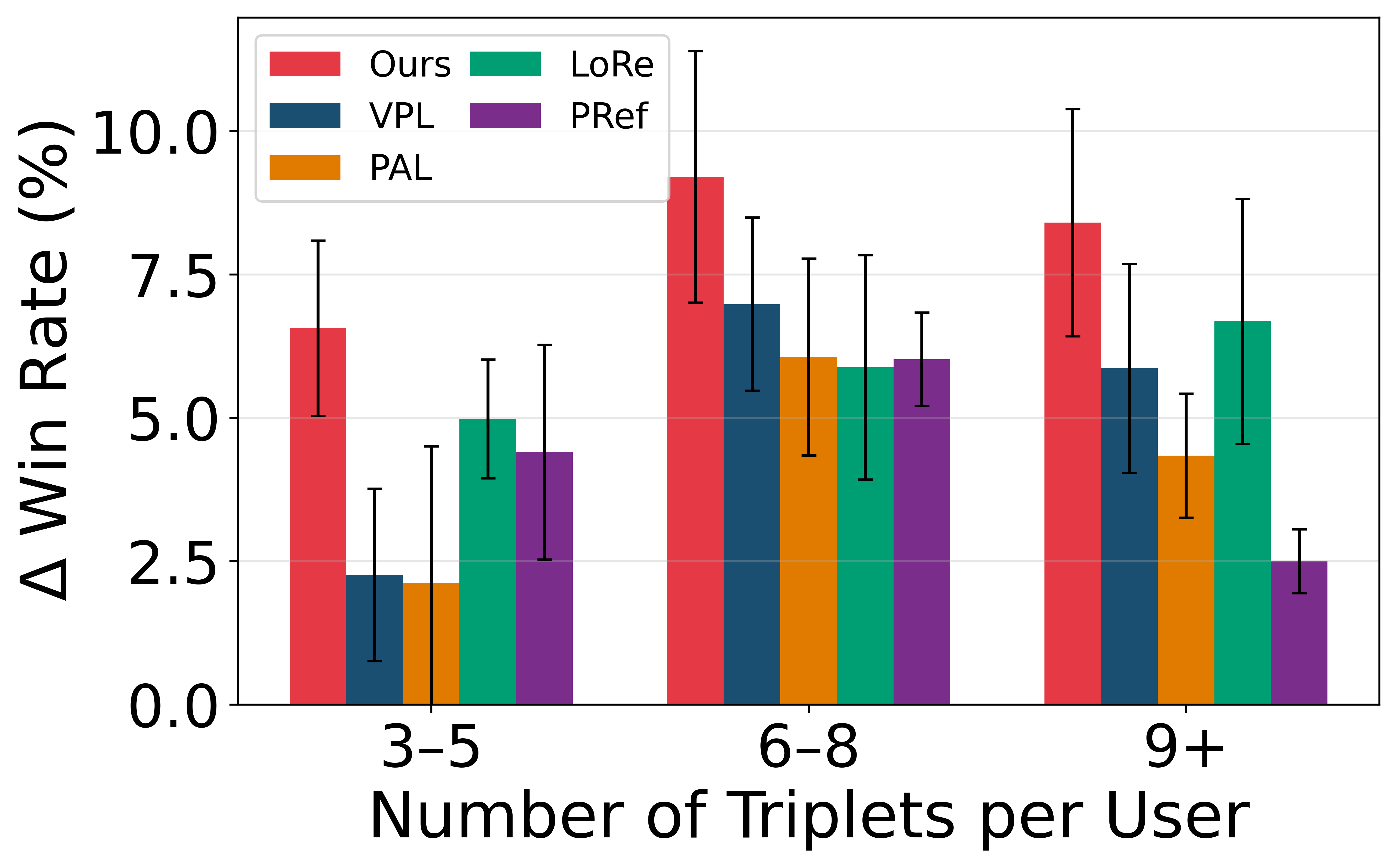}
        \subcaption{Inference-time alignment (grouped by $|\mathcal{C}_u|$)}
        \label{fig:winrate}
    \end{minipage}
    \hfill
    \begin{minipage}[b]{0.48\linewidth}
        \centering
        \includegraphics[width=\linewidth]{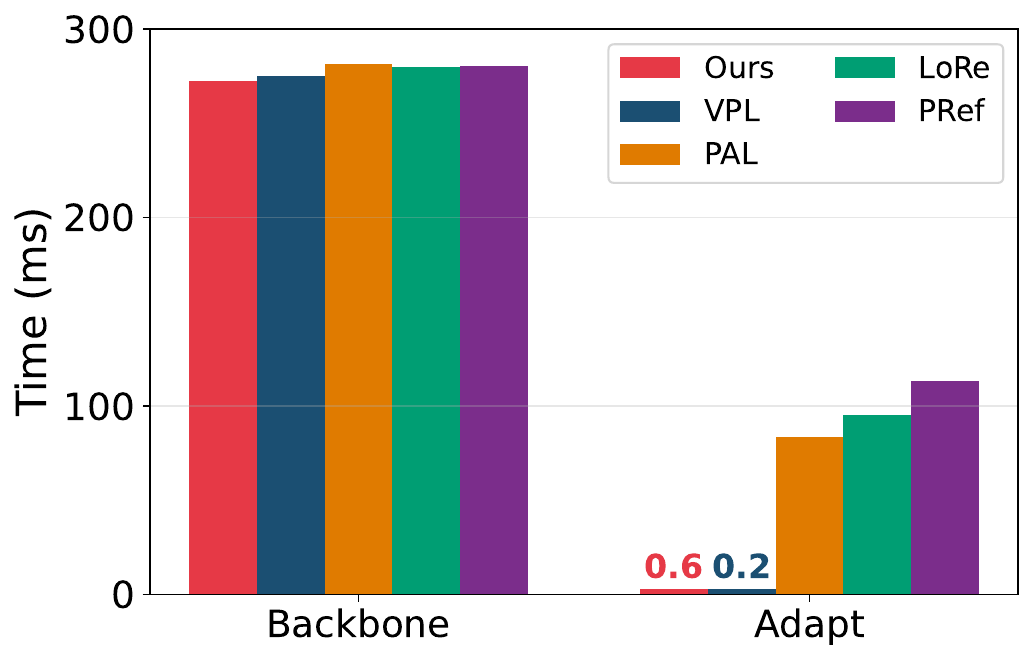}
        \subcaption{Per-user adaptation time (ms)}
        \label{fig:adpat_time}
    \end{minipage}

    \caption{In-depth analysis of unseen users in PRISM. \proposed maintains superior accuracy across all user groups (a--c) while significantly reducing adaptation latency (d).}
    \label{fig:combined}
\end{figure*}

\begin{table}[t]
\centering
\caption{Calibration of predicted preference probabilities on the overall split.
Lower is better for all metrics; ECE uses ten equal-width bins.}
\label{tab:calibration}
\small
\setlength{\tabcolsep}{6pt}
\begin{tabular}{l l | ccc}
\toprule
\textbf{Dataset} & \textbf{Model}
& \textbf{NLL} ($\downarrow$) & \textbf{Brier} ($\downarrow$) & \textbf{ECE} ($\downarrow$) \\
\midrule\midrule
\multirow{2}{*}{PersonalLLM ($\alpha{=}0.001$)}
  & BT & 0.2584\std{0.0069} & 0.0734\std{0.0017} & 0.0162\std{0.0034} \\
  & \proposed & \textbf{0.1747}\std{0.0195} & \textbf{0.0464}\std{0.0027} & \textbf{0.0099}\std{0.0095} \\
\midrule

\multirow{2}{*}{PRISM}
  & BT & 0.6862\std{0.0071} & 0.2402\std{0.0027} & 0.0844\std{0.0075} \\
  & \proposed & \textbf{0.6348}\std{0.0049} & \textbf{0.2237}\std{0.0018} & \textbf{0.0269}\std{0.0127} \\
\midrule

\multirow{2}{*}{TL;DR ($n{=}100$)}
  & BT & 0.7572\std{0.0399} & 0.2600\std{0.0086} & 0.1353\std{0.0292} \\
  & \proposed & \textbf{0.6862}\std{0.0310} & \textbf{0.2394}\std{0.0099} & \textbf{0.0761}\std{0.0348} \\
\midrule
\multirow{2}{*}{TL;DR ($n{=}150$)}
  & BT & 0.7277\std{0.0488} & 0.2499\std{0.0103} & \textbf{0.1155}\std{0.0356} \\
  & \proposed & \textbf{0.7203}\std{0.0183} & \textbf{0.2452}\std{0.0060} & 0.1216\std{0.0087} \\
\bottomrule
\end{tabular}
\end{table}
\paragraph{Uncertainty Calibration.}
To examine whether \proposed's uncertainty-aware modeling translates into
well-calibrated preference predictions, we compare it against BT on three
calibration metrics: NLL \citep{gneiting2007strictly}, Brier score
\citep{brier1950verification}, and ECE \citep{guo2017calibration}.
BT predicts a deterministic preference probability (Eq.~\ref{eq:r_bt}), whereas \proposed attenuates it by the reward variance (Eq.~\ref{eq:l_vbt}).
Table~\ref{tab:calibration} shows that \proposed achieves better calibration than BT across most settings, since
attenuating the prediction under high variance prevents overconfidence
when a user's preference is inferred from limited observations.
The gain is large on PersonalLLM ($\alpha{=}0.001$) and PRISM, where ECE drops
by 39\% and 68\%. These are the settings with the most diverse preferences and
the fewest observations per user.
A controlled comparison between the two TL;DR settings, which share the same users and differ only in the number of training triplets per user, shows that the gain is larger under sparser data ($n{=}100$).
Overall, \proposed's uncertainty-aware preference modeling yields more reliable
predictions than deterministic point estimates (C2).

\subsection{Study of \proposed}
\label{subsec:study of ours}
\paragraph{Ablation Study.}
\label{sec:ablation}

Table~\ref{tab:ablation} reports the contribution of each component in \proposed.
(a) Removing $\mathcal{L}_{\text{VBT}}$ (replaced with $\mathcal{L}_{\text{BT}}$)
consistently degrades accuracy across all datasets, confirming its role in adaptively scaling the gradients under preference uncertainty.
(b) Removing $\mathcal{L}_{\mathrm{SEP}}$ causes the largest drop on PersonalLLM ($-3.6\%$), indicating that basis separation is critical for capturing distinct preferences.
(c) Removing $\mathcal{L}_{\mathrm{REG}}$ results in smaller but consistent degradation across all datasets.
(d) Replacing the mixture-of-Gaussians (MoG) prior with $\mathcal{N}(\mathbf{0}, \mathbf{I})$ degrades performance on all datasets, 
highlighting the benefit of regularizing user distributions to lie within the learned preference space.
(e) Replacing the product-of-Gaussians (PoG) with mean pooling degrades performance, showing that uncertainty-aware aggregation of user observations yields more reliable estimation.
Since preference bases are integral to \proposed and cannot be ablated, we analyze sensitivity to the number of bases ($K$), along with the regularization strength ($\beta$), in Appendix~\ref{appendix:hyperparams}.

\begin{table*}[t]
\centering
\caption{Ablation study on \proposed across three datasets.}
\label{tab:ablation}
\small
\setlength{\tabcolsep}{3.5pt}
\resizebox{\textwidth}{!}{
\begin{tabular}{l | ccc | ccc | ccc}
\toprule
\multirow{2}{*}{\textbf{Model Variant}} & \multicolumn{3}{c|}{\textbf{PersonalLLM ($\alpha{=}0.001$)}}
& \multicolumn{3}{c|}{\textbf{TL;DR ($n{=}150$)}}
& \multicolumn{3}{c}{\textbf{PRISM}} \\
& {Seen} & {Unseen} & {Overall} & {Seen} & {Unseen} & {Overall} & {Seen} & {Unseen} & {Overall} \\
\midrule\midrule

(a) w/o $\mathcal{L}_{\text{VBT}}$
  & 92.7\std{0.3} & 92.5\std{0.7} & 92.6\std{0.3}
  & 61.4\std{0.5} & 60.8\std{0.3} & 61.1\std{0.4}
  & 58.3\std{0.2} & 58.1\std{0.5} & 58.2\std{0.4} \\
(b) w/o $\mathcal{L}_{\text{SEP}}$
  & 91.9\std{0.3} & 91.4\std{0.7} & 91.7\std{0.3}
  & 61.5\std{0.6} & 60.8\std{0.3} & 61.2\std{0.5}
  & 58.2\std{0.1} & 58.2\std{0.5} & 58.2\std{0.3} \\
(c) w/o $\mathcal{L}_{\text{REG}}$
  & 93.2\std{0.2} & 93.5\std{0.3} & 93.3\std{0.2}
  & 61.2\std{0.6} & 60.3\std{0.4} & 60.8\std{0.4}
  & 60.8\std{0.4} & 60.8\std{0.9} & 60.8\std{0.7} \\
(d) w/o MoG 
  & 91.9\std{0.3} & 91.4\std{0.7} & 91.7\std{0.2}
  & 61.3\std{0.5} & 60.8\std{0.3} & 61.1\std{0.4}
  & 60.2\std{0.3} & 60.3\std{0.8} & 60.2\std{0.6} \\
(e) w/o PoG
  & 94.8\std{0.2} & 95.0\std{0.2} & 94.9\std{0.2}
  & 62.7\std{0.2} & 61.8\std{0.4} & 62.2\std{0.3}
  & 61.7\std{0.5} & 60.6\std{0.4} & 61.1\std{0.2} \\
\midrule
\textbf{Full Model}
  & \textbf{95.6}\std{0.4} & \textbf{95.1}\std{0.3} & \textbf{95.3}\std{0.3}
  & \textbf{63.3}\std{1.1} & \textbf{62.4}\std{1.1} & \textbf{62.9}\std{1.1}
  & \textbf{62.1}\std{0.6} & \textbf{61.6}\std{0.5} & \textbf{61.8}\std{0.3} \\
\bottomrule
\end{tabular}
}
\vspace{-0.3cm}
\end{table*}

\label{sec:}
\begin{figure}
    \hspace{-0.2cm}
    \includegraphics[width=\linewidth]{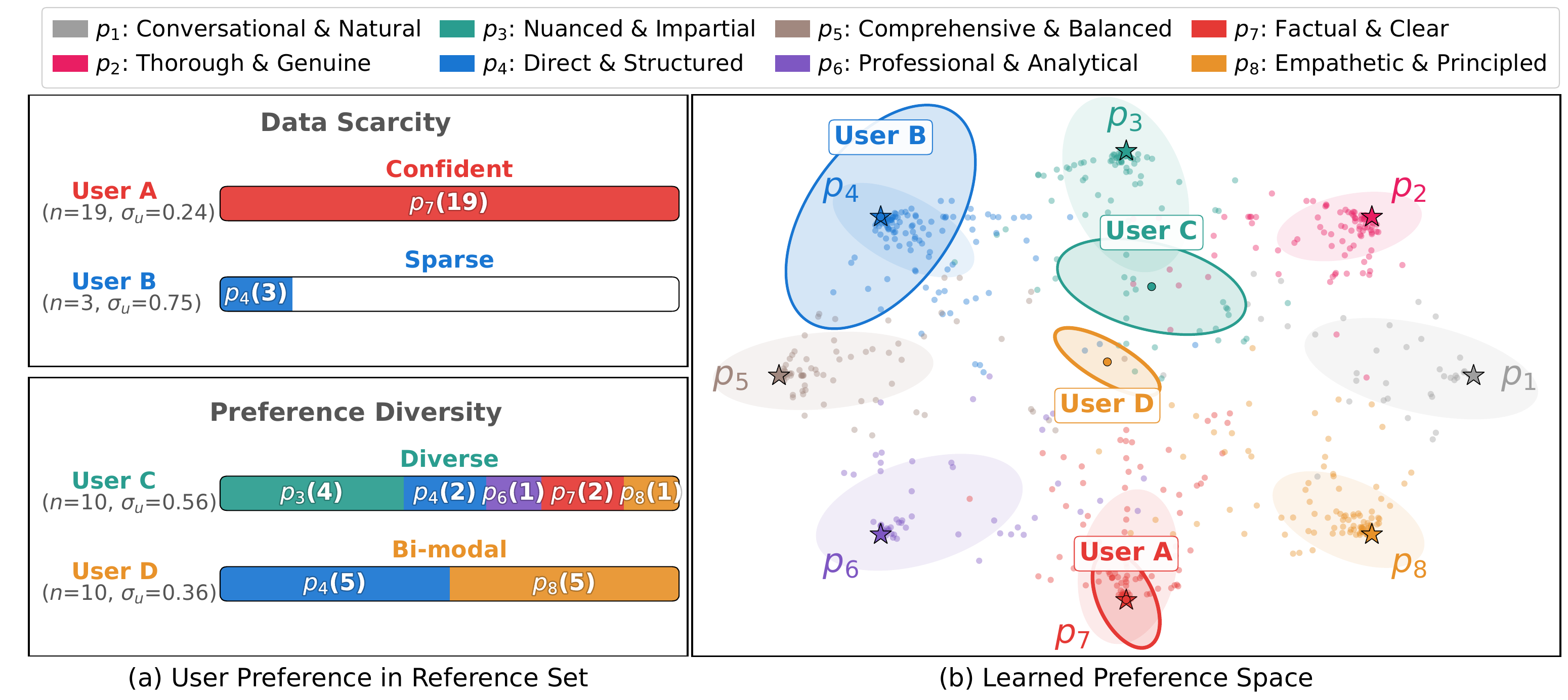}
   \caption{Case studies and the learned preference space on PRISM ($K{=}8$). 
   (a) Four users with varying data scarcity and preference diversity. $n$ is the reference set size.
   Bars show the preference distribution classified by Claude Sonnet 4.6 (see Appendix~\ref{app:prompt_preference} for details).
   (b) Learned preference space. Stars and shaded regions denote the preference bases.
   For user ellipses, the center point represents $\mathbf{w}_u$, and the size reflects ${\sigma}_u$.}
    \label{fig:pref_space}
\end{figure}
\paragraph{Case Studies on Learned Preference Space.}
Figure~\ref{fig:pref_space}a shows four users across two axes. 
In terms of data scarcity, both User~A ($n{=}19$) and User~B ($n{=}3$) exhibit consistent preferences, but the limited observations for User~B result in higher uncertainty ($\sigma_u{=}0.75$).
In terms of preference diversity, User~C and User~D have the same reference set size ($n{=}10$), yet User~C spreads across five bases while User~D concentrates on two ($p_4$, $p_8$), resulting in higher uncertainty for User~C ($\sigma_u{=}0.56$).

Figure~\ref{fig:pref_space}b shows how these users' preference distributions are represented in the learned preference space. 
User A exhibits a tight ellipse near $p_7$, reflecting high confidence from abundant, consistent data. 
User B, despite concentrating on $p_4$, shows a large ellipse due to data scarcity. 
User C spans multiple bases, capturing diverse preferences.
User D sits between $p_4$ and $p_8$ with moderate uncertainty. 
This visualization shows that \proposed effectively disentangles what a user prefers from how confident the estimate~is.

\section{Conclusion}
\label{sec:conclusion}
In this paper, we present \proposed for uncertainty-aware LLM personalization, addressing two key challenges in reward factorization.
For \textbf{C1}, \proposed introduces shared preference bases, allowing data-scarce and unseen users to leverage collective preference patterns.
For \textbf{C2}, VRF represents both users and bases as Gaussians for uncertainty-aware distributional matching, further enhanced by variance-attenuated training.
Experiments on three benchmarks show consistent improvements over all baselines, robustness to scarcity and uncertainty, and gains in downstream alignment. 
We believe \proposed provides an effective and scalable path for reliable LLM personalization.
Future directions include query-adaptive personalization across diverse contexts and temporal adaptation to evolving user preferences.




\section*{Acknowledgments}
This work used Delta at the National Center for Supercomputing Applications
through allocation CIS260153 from the Advanced Cyberinfrastructure
Coordination Ecosystem: Services \& Support (ACCESS)
program~\citep{boerner2023access}, which is supported by U.S. National Science
Foundation grants \#2138259, \#2138286, \#2138307, \#2137603, and \#2138296.
Jiawei Han's research is supported by NSF-2505932, NSF IIS 25-37827, and NSF-2118329.

\bibliography{colm2026_conference}
\bibliographystyle{colm2026_conference}

\newpage
\appendix
\section{Experimental Setup}
\subsection{Datasets}
\label{appendix:datasets}
\begin{table}[t]
\centering

\caption{Average dataset statistics per user. For seen users, $C_u$ is sampled from training data per epoch, and the full training set serves as $C_u$ at evaluation. For unseen users, only $C_u$ is provided at inference to estimate user weights. Both settings are evaluated on held-out $D_u$.}
\label{tab:dataset_stats}
\begin{tabular}{l | ccccc | ccc}
\toprule
\multirow{3}{*}{\textbf{Dataset}} & \multicolumn{5}{c|}{\textbf{Seen}} & \multicolumn{3}{c}{\textbf{Unseen}} \\
\cmidrule(lr){2-6}\cmidrule(lr){7-9}
& \multirow{2}{*}{\#Users} & \multicolumn{2}{c}{Training} & \multicolumn{2}{c|}{Evaluation} & \multirow{2}{*}{\#Users} & \multicolumn{2}{c}{Evaluation}  \\
&  & {$|\mathcal{C}_u|$} & {$|\mathcal{D}_u|$} & {$|\mathcal{C}_u|$} & {$|\mathcal{D}_u|$} &  & {$|\mathcal{C}_{u}|$} & {$|\mathcal{D}_{u}|$} \\
\midrule\midrule
PersonalLLM        & 1{,}000 & 9.0   & 36.0  & 45.0  & 45.0      & 1{,}000 &  9.0 &  9.0 \\
TL;DR ($n{=}100$)  &      20 & 50.0  & 50.0  & 100.0 & 1{,}449.8 &      20 & 50.0 & 1{,}146.4 \\
TL;DR ($n{=}150$)  &      20 & 50.0  & 100.0 & 150.0 & 1{,}449.8 &      20 & 50.0 & 1{,}146.4 \\
PRISM              &     623 & 3.0   & 6.1   & 9.1   & 8.6       &     623 &  9.6 &  9.1 \\
\bottomrule
\end{tabular}
\label{tab:dataset}
\end{table}

Following state-of-the-art reward factorization work~\citep{PAL, LoRe, PReF}, we evaluate on three benchmarks spanning synthetic and real-world personalized preference data. Detailed statistics are provided in Table~\ref{tab:dataset}.
\begin{itemize}[leftmargin=*] \vspace{-\topsep}
    \item
    \textbf{PersonalLLM}~\citep{personalllm} provides 9,402 prompts, each with 8 candidate responses scored by 10 reward models (i.e., $r \in \mathbb{R}^{10}$). 
    To construct personalized preference pairs, we synthesize a user preference vector $\mathbf{w} \sim \mathrm{Dirichlet}(\alpha) \in \mathbb{R}^{10}$ and compute a composite score $s = \mathbf{w}^\top \mathbf{r}$.
    For each prompt, we form a preference pair from the highest- and lowest-scoring responses.
    We synthesize 2,000 users, equally split into seen and unseen.
    We evaluate across three diversity levels, $\alpha \in \{0.001, 0.01, 0.1\}$, where a smaller $\alpha$ produces more heterogeneous preferences, making personalization more critical.

    \item
    \textbf{TL;DR}~\citep{stiennon2020learning} contains pairwise human preference data over Reddit post summaries, organized by individual annotators.
    We retain annotators with at least 50 annotations and randomly select 40, partitioning them into 20 seen and 20 unseen users. 
    Seen users have $n \in \{100, 150\}$ triplets for training.
    \item
    \textbf{PRISM}~\citep{prism} collects multi-turn conversations between crowdworkers and LLMs, where annotators score multiple model responses per interaction. We use the demographically balanced subset (623 seen / 623 unseen users) and construct pairwise preference data from each interaction using the highest- and lowest-scored responses.
\end{itemize}

\subsection{Baselines}
\label{appendix:baselines}
We compare against state-of-the-art personalized reward modeling methods. 
To ensure a fair comparison, baselines include $\mathcal{C}_u$ in the training set for seen users.
\begin{itemize}[leftmargin=*] \vspace{-\topsep}
    \item \textbf{Ref}~\citep{qwen3} is a reference baseline that uses the pretrained LLM's per-token log-probability as a reward signal (i.e., $r(x,y) = \frac{1}{|y|}\sum_i \log P(y_i \mid x, y_{<i})$),  without any training or personalization.


    \item \textbf{BT}~\citep{stiennon2020learning, ouyang2022training} is a standard Bradley–Terry~\citep{bradley1952rank} reward model fine-tuned on all users' preference data jointly, without any user-specific personalization.
    
    \item\textbf{VPL}~\citep{VPL} encodes each user's preference history into a variational latent variable, which conditions the reward prediction.

    \item\textbf{PAL}~\citep{PAL} represents each user as an ideal point in a shared preference latent space to compute personalized rewards.
    
    \item\textbf{LoRe}~\citep{LoRe} models individual user preferences as weighted combinations of shared basis reward functions.
    
    \item\textbf{PReF}~\citep{PReF} represents user-specific rewards as a linear combination of basis reward functions with SVD initialization and L2 regularization.
    
\end{itemize}

\subsection{Implementation Details}
\begin{table}[t]
\centering
\caption{Best hyperparameter configurations for \proposed across all datasets.}
\label{tab:best_hp}
\resizebox{0.9\textwidth}{!}{
\begin{tabular}{l | cccccc}
\toprule
\textbf{Dataset} & \textbf{Epochs} & \textbf{Batch size} & \textbf{Learning rate}  & \textbf{Dropout} & \textbf{Weight decay} & $K$ \\
\midrule\midrule
PersonalLLM ($\alpha{=}0.001$) & 5 & 8 & 2e-5 & 0.0 & 0.01  & 8  \\
PersonalLLM ($\alpha{=}0.01$)  & 5 & 8 & 2e-5 & 0.0 & 0.001 & 8  \\
PersonalLLM ($\alpha{=}0.1$)   & 5 & 8 & 2e-5 & 0.0 & 0.01  & 8  \\
TL;DR ($n{=}100$)              & 3 & 8 & 6e-5 & 0.1 & 0.001 & 8  \\
TL;DR ($n{=}150$)              & 3 & 8 & 6e-5 & 0.1 & 0.001 & 8  \\
PRISM                          & 5 & 8 & 6e-5 & 0.0 & 0.01  & 12 \\
\bottomrule
\end{tabular}
}
\end{table}

\label{appendix:implementation_details}
All experiments are conducted on NVIDIA A40 GPUs and AMD EPYC 7763 CPUs.
We optimize with AdamW and cosine learning rate scheduling.
We fix the latent dimension $D{=}64$, and regularization strength $\beta{=}0.001$ across all datasets, as performance is insensitive to these values.
We perform grid search over learning rate $\in \{2\mathrm{e-}5, 4\mathrm{e-}5, 6\mathrm{e-}5\}$, dropout $\in \{0.0, 0.1\}$, weight decay $\in \{0.001, 0.01\}$, $K \in \{4, 8, 12, 16, 20\}$, temperature hyperparameters $\tau_d,\tau_m\in\{0.1,1.0\}$.
The best hyperparameters per dataset are reported in Table~\ref{tab:best_hp}.

\subsection{Inference-Time Alignment Setup}
\label{appendix:inf_setup}
Figure~\ref{fig:winrate} shows the inference-time alignment results via best-of-$N$ sampling on PRISM unseen users.
For each user's test query, we generate $N{=}8$ candidate responses using GPT-5 mini as the policy model, each conditioned on a distinct style inferred from the preference basis label (see Figure~\ref{fig:pref_space}).
Following the LLM-as-a-judge framework~\citep{zheng2023judging}, GPT-5.1 evaluates each method's selected response against BT's in a pairwise manner, with the user's profile (demographics, preferences, and values) as context and randomized response ordering to mitigate position bias.
We report the $\Delta$ Win Rate, the win rate over BT normalized by the global minimum.
Full prompt templates are provided in Appendix~\ref{app:prompt_inference_time_alignment}.

\section{Additional Experiments}
\subsection{Backbone Scaling}
\label{app:backbone_scaling}


\begin{table}[t]
\centering
\caption{Effect of backbone scale on accuracy and efficiency for the PRISM dataset.}
\small
\label{tab:backbone}
\resizebox{0.9\textwidth}{!}{
\begin{tabular}{l|ccccc}
\toprule
\textbf{Backbone} & \textbf{Seen} & \textbf{Unseen} & \textbf{Overall} & \textbf{Training (min)} & \textbf{Inference (min)} \\
\midrule\midrule
0.6B & 61.31{\scriptsize$\pm$0.57} & 60.58{\scriptsize$\pm$0.53} & 60.94{\scriptsize$\pm$0.18} & \textbf{80.8} & \textbf{20.1} \\
1.7B & 61.28{\scriptsize$\pm$0.32} & 60.44{\scriptsize$\pm$0.27} & 60.86{\scriptsize$\pm$0.21} & 131.4 & 33.6 \\
4B   & 61.72{\scriptsize$\pm$0.51} & \textbf{61.31{\scriptsize$\pm$0.49}} & 61.51{\scriptsize$\pm$0.43} & 288.1 & 74.6 \\
8B   & \textbf{62.14{\scriptsize$\pm$0.46}} & 61.15{\scriptsize$\pm$0.46} & \textbf{61.64{\scriptsize$\pm$0.32}} & 446.8 & 117.3 \\
\bottomrule
\end{tabular}
}
\end{table}
Table~\ref{tab:backbone} examines the effect of backbone scale by varying Qwen3 from 0.6B to 8B parameters with LoRA~\citep{hu2022lora} fine-tuning ($r{=}16$, $\alpha{=}32$) under identical hyperparameters.
While the 8B backbone achieves the best overall accuracy, the improvement over 0.6B is marginal (+0.7\%) relative to its cost (5.5× training and 5.8× inference time).
This accuracy--efficiency trade-off suggests that smaller backbones are a practical choice for \proposed, as shared preference bases and uncertainty-aware distributional matching compensate for reduced backbone capacity.




\begin{figure*}[t]
    \centering

    \begin{minipage}[b]{\linewidth}
        \centering
        \includegraphics[width=0.48\linewidth]{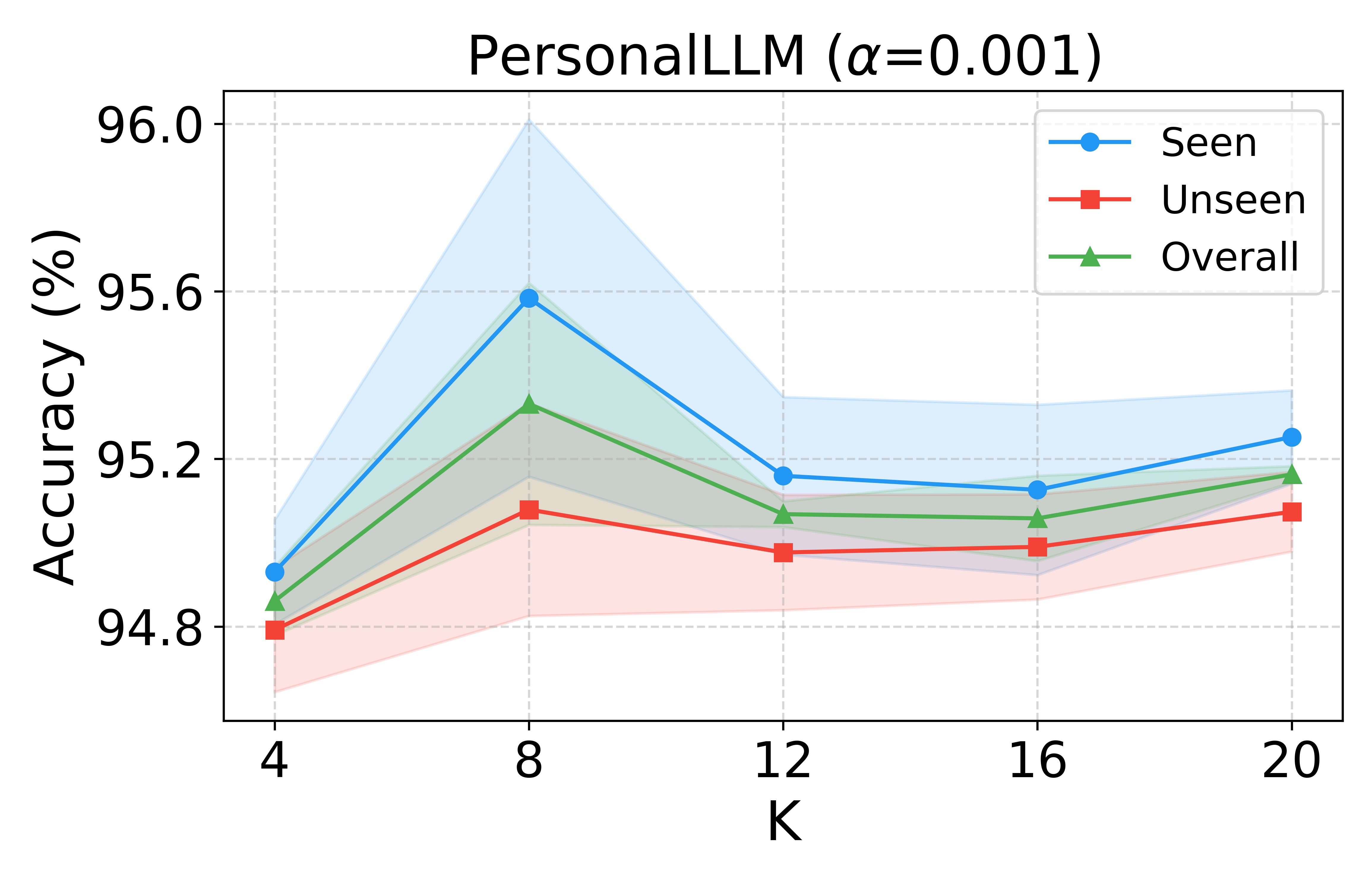}
        \hfill
        \includegraphics[width=0.48\linewidth]{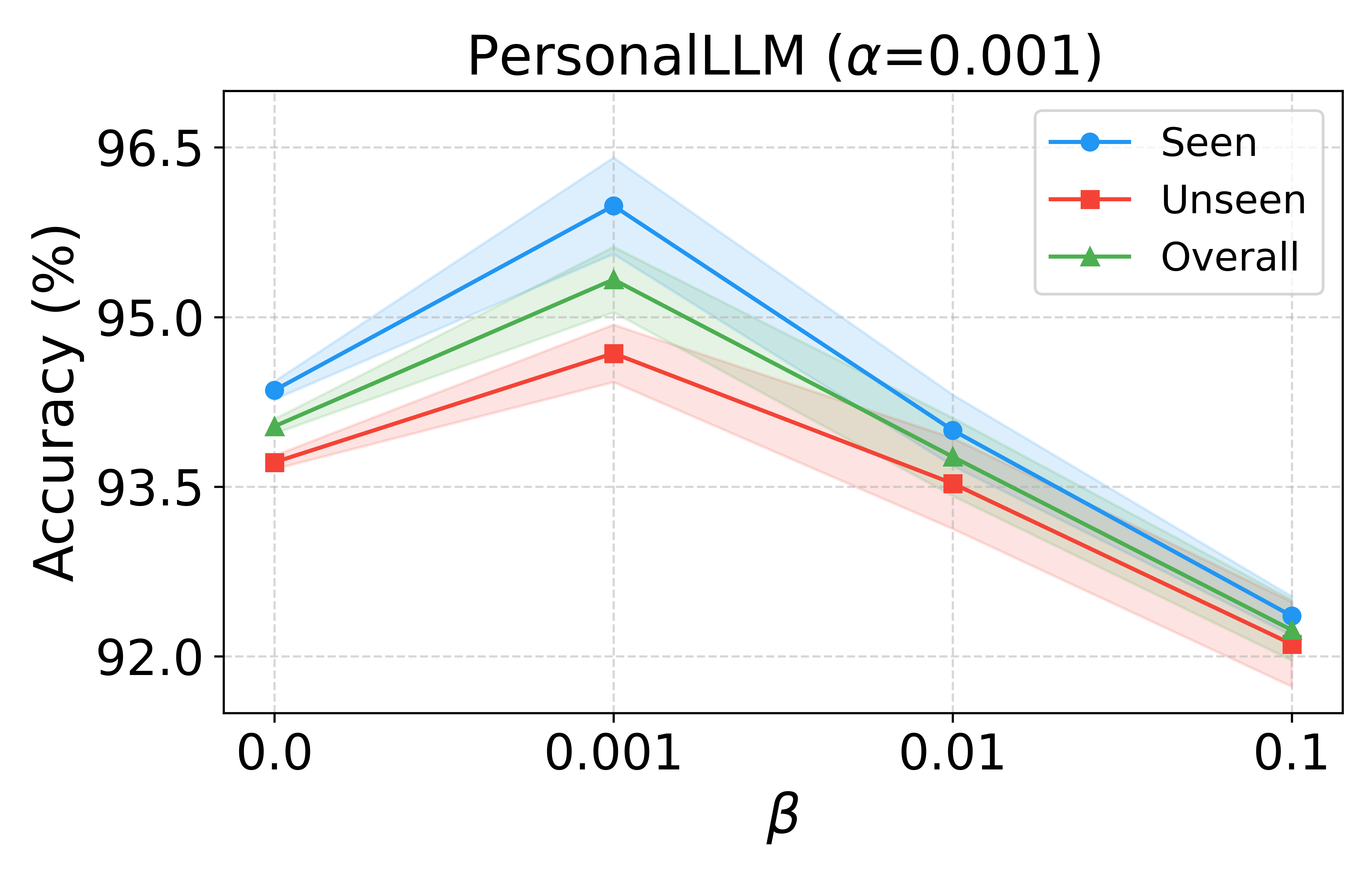}
    \end{minipage}
    
    \vspace{5pt}

    \begin{minipage}[b]{\linewidth}
        \centering
        \includegraphics[width=0.48\linewidth]{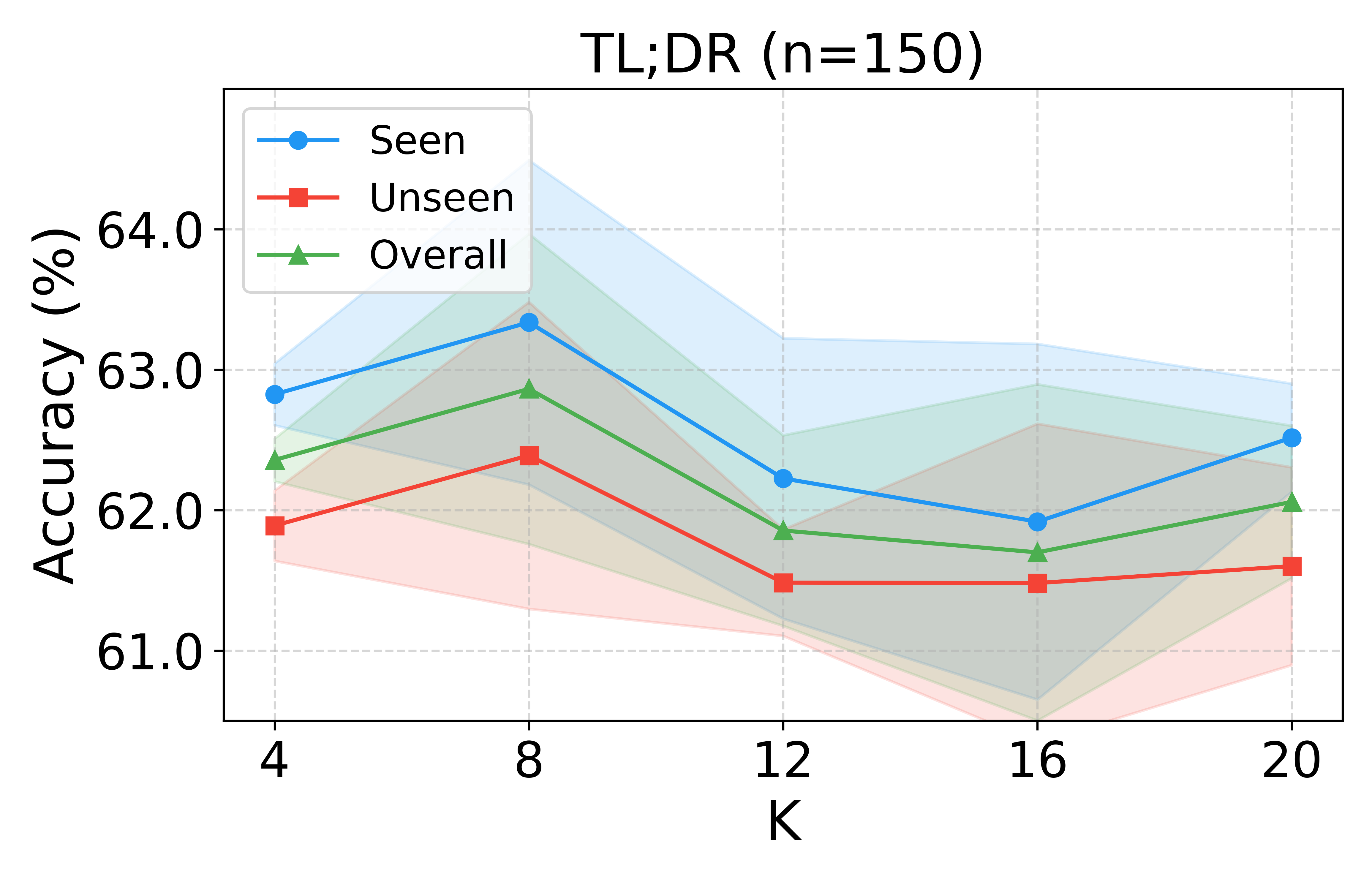}
        \hfill
        \includegraphics[width=0.48\linewidth]{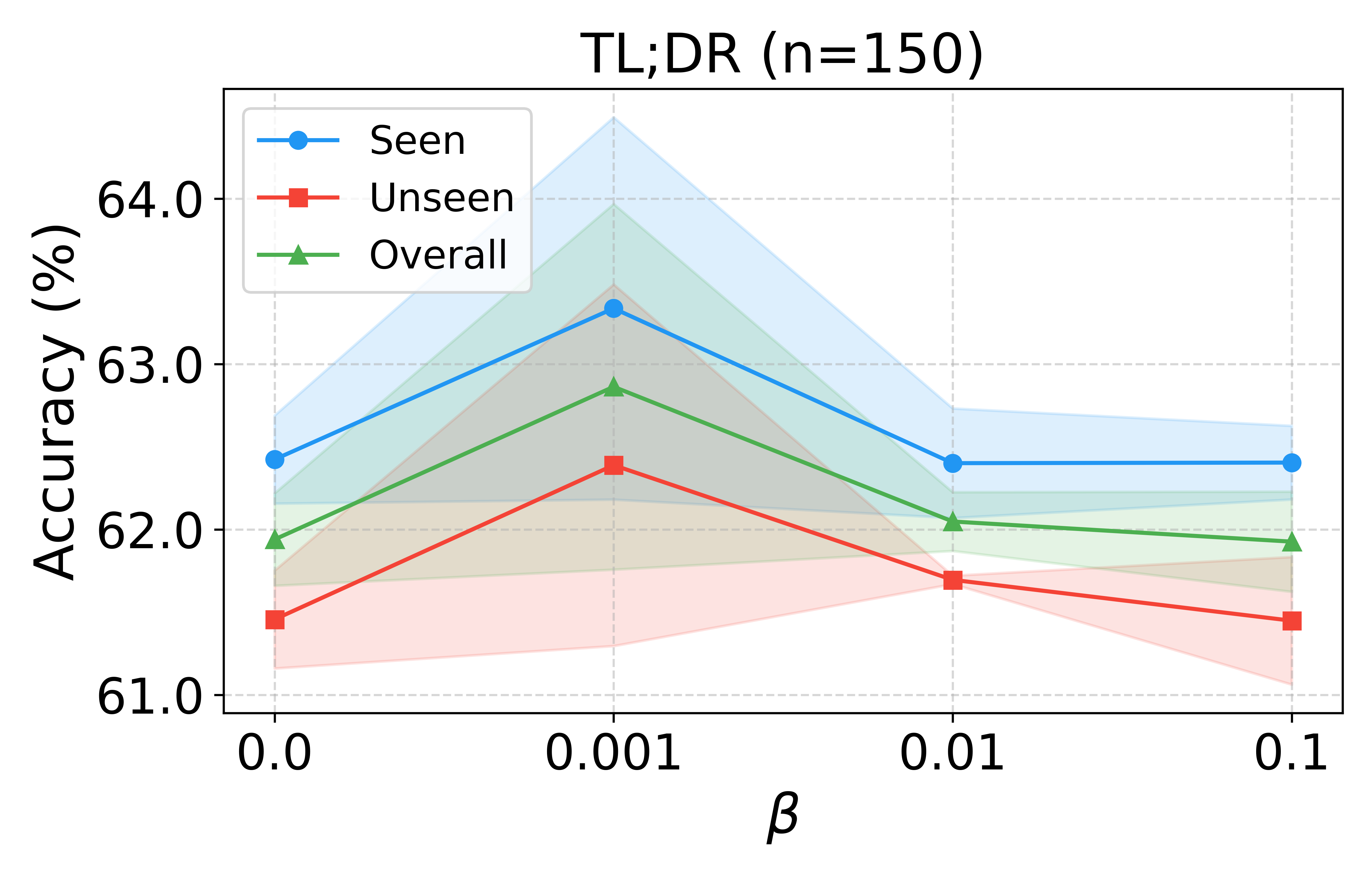}
    \end{minipage}

    \vspace{5pt}

    \begin{minipage}[b]{\linewidth}
        \centering
        \includegraphics[width=0.48\linewidth]{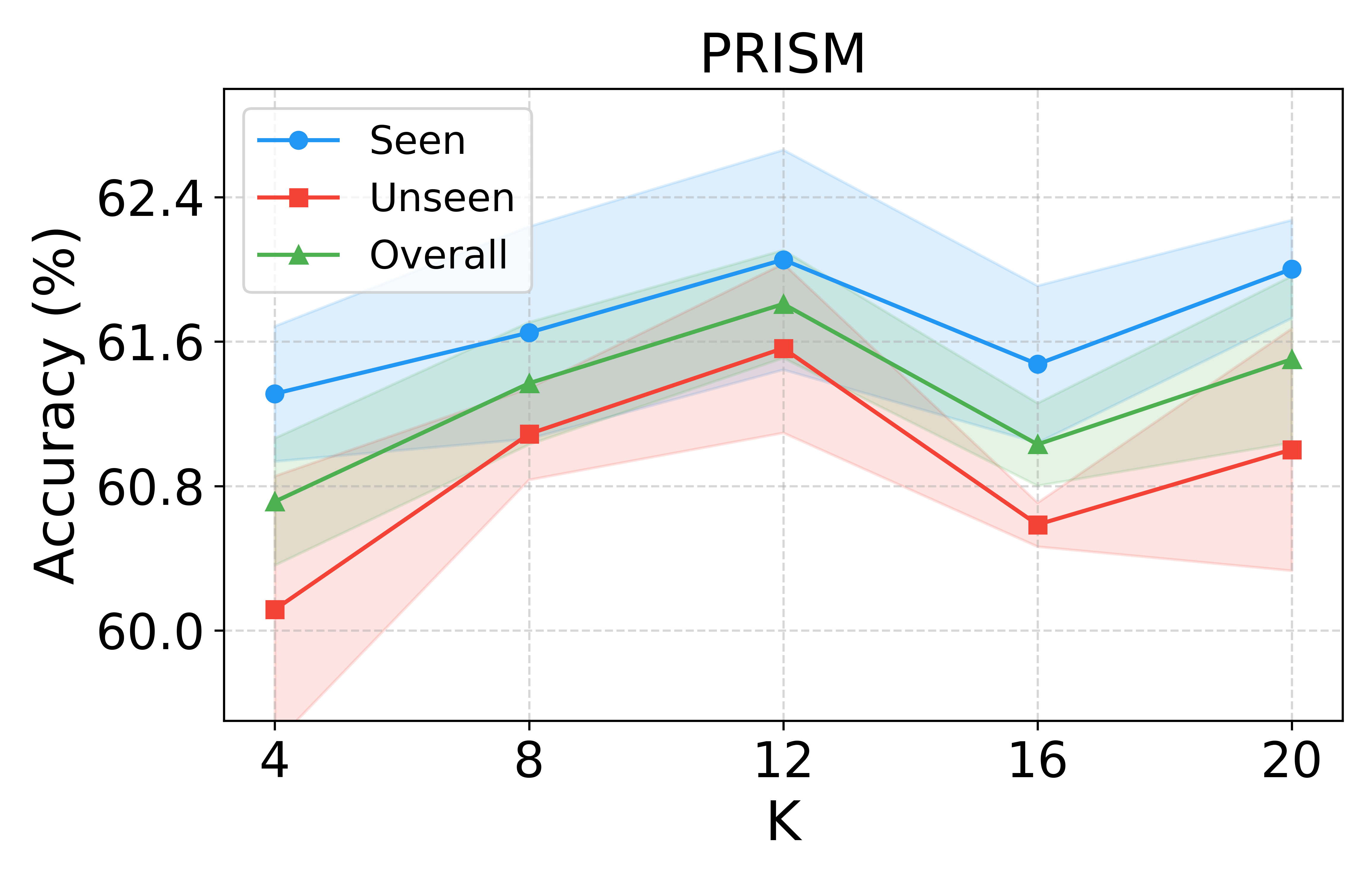}
        \hfill
        \includegraphics[width=0.48\linewidth]{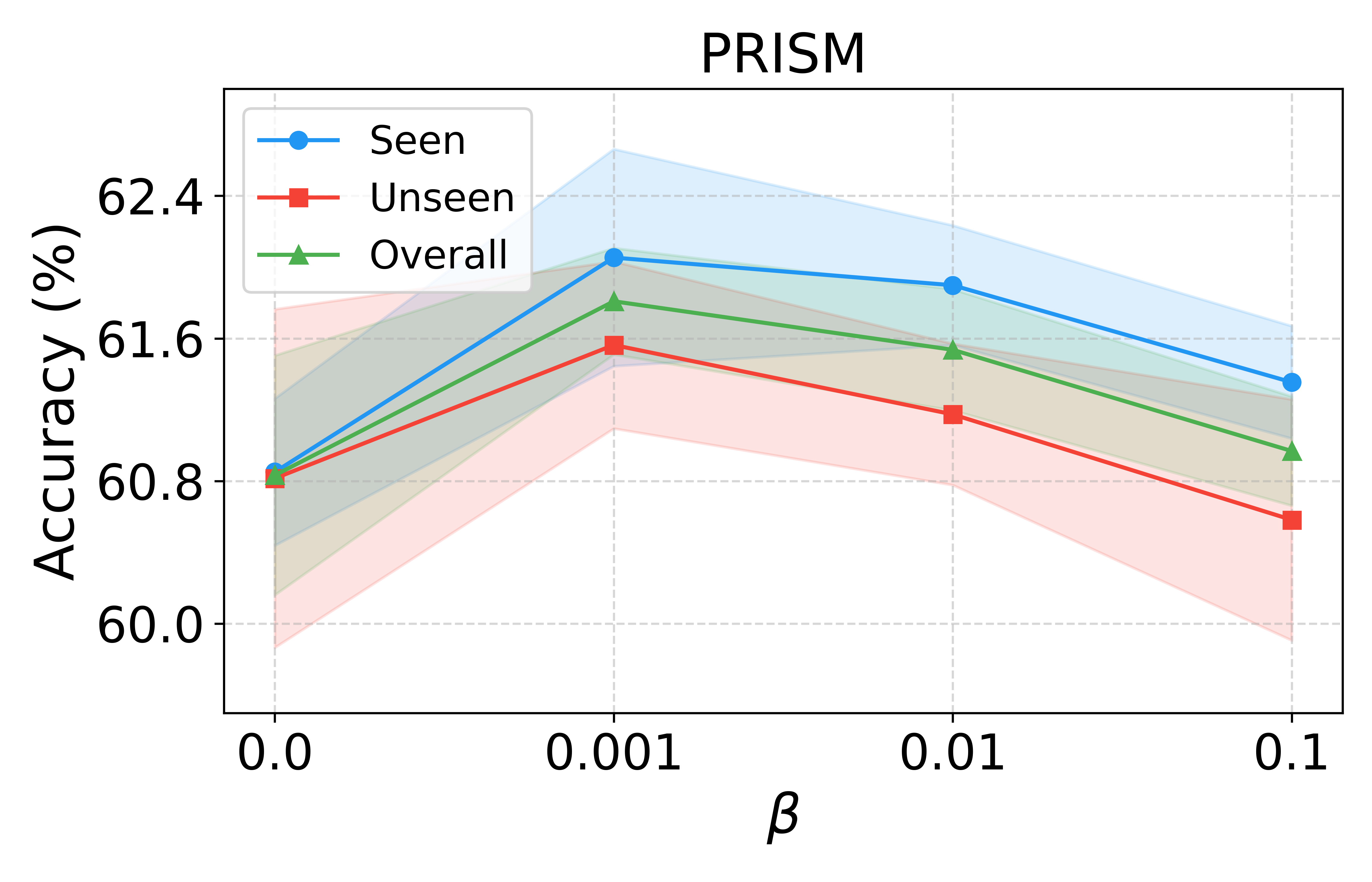}
    \end{minipage}
    \caption{Hyperparameter sensitivity analysis on three datasets.} 
    \label{fig:hyperparameters}
\end{figure*}
\subsection{Hyperparameter Sensitivity}
\label{appendix:hyperparams}

\paragraph{Number of preference bases ($K$).}
Figure~\ref{fig:hyperparameters} (left) shows the effect of varying $K \in \{4, 8, 12, 16, 20\}$ on three datasets.
Performance peaks at K=8 on PersonalLLM and TL;DR and at K=12 on PRISM, with marginal declines beyond.
These results suggest that too many bases may introduce redundancy or optimization difficulty. 
In other words, a small number of shared preference bases suffices to capture population-level preference diversity, making \proposed scalable as $K$ remains fixed regardless of user population size.



\paragraph{Regularization strength ($\beta$).}
Figure~\ref{fig:hyperparameters} (right) shows the effect of varying 
$\beta \in \{0.0, 0.001, 0.01, 0.1\}$. Across all three datasets, 
performance peaks at $\beta{=}0.001$ and degrades at both extremes. 
Without regularization ($\beta{=}0.0$), user distributions lack structural guidance from the mixture-of-Gaussians prior, scattering arbitrarily in the latent space.
Conversely, excessive regularization ($\beta{=}0.1$) over-concentrates user distributions toward the prior, potentially suppressing user-specific information.
We adopt $\beta{=}0.001$ as the default.


\subsection{Few-shot, Uncertainty and Efficiency Analysis on Additional Datasets}
\label{appendix:uncertain_fewshot}
Figure~\ref{fig:combined_extra} extends the few-shot adaptation, uncertainty robustness and adaptation efficiency analysis of Figure~\ref{fig:combined} to PersonalLLM ($\alpha{=}0.001$) and TL;DR ($n{=}150$).
\proposed achieves the highest accuracy across few-shot settings and all uncertainty groups while maintaining near-zero adaptation overhead, confirming that the findings on PRISM generalize across datasets.
Notably, in few-shot adaptation (Figure~\ref{fig:fewshot_2}), LoRe and PReF show increasing accuracy as $|\mathcal{C}_u|$ grows, whereas \proposed achieves the highest accuracy even with $|\mathcal{C}_u|{=}1$, highlighting that the variational user encoder effectively infers the preferences of unseen users.

\begin{figure*}[t]
    \centering

    \begin{minipage}[b]{\linewidth}
        \centering
        \includegraphics[width=0.48\linewidth]{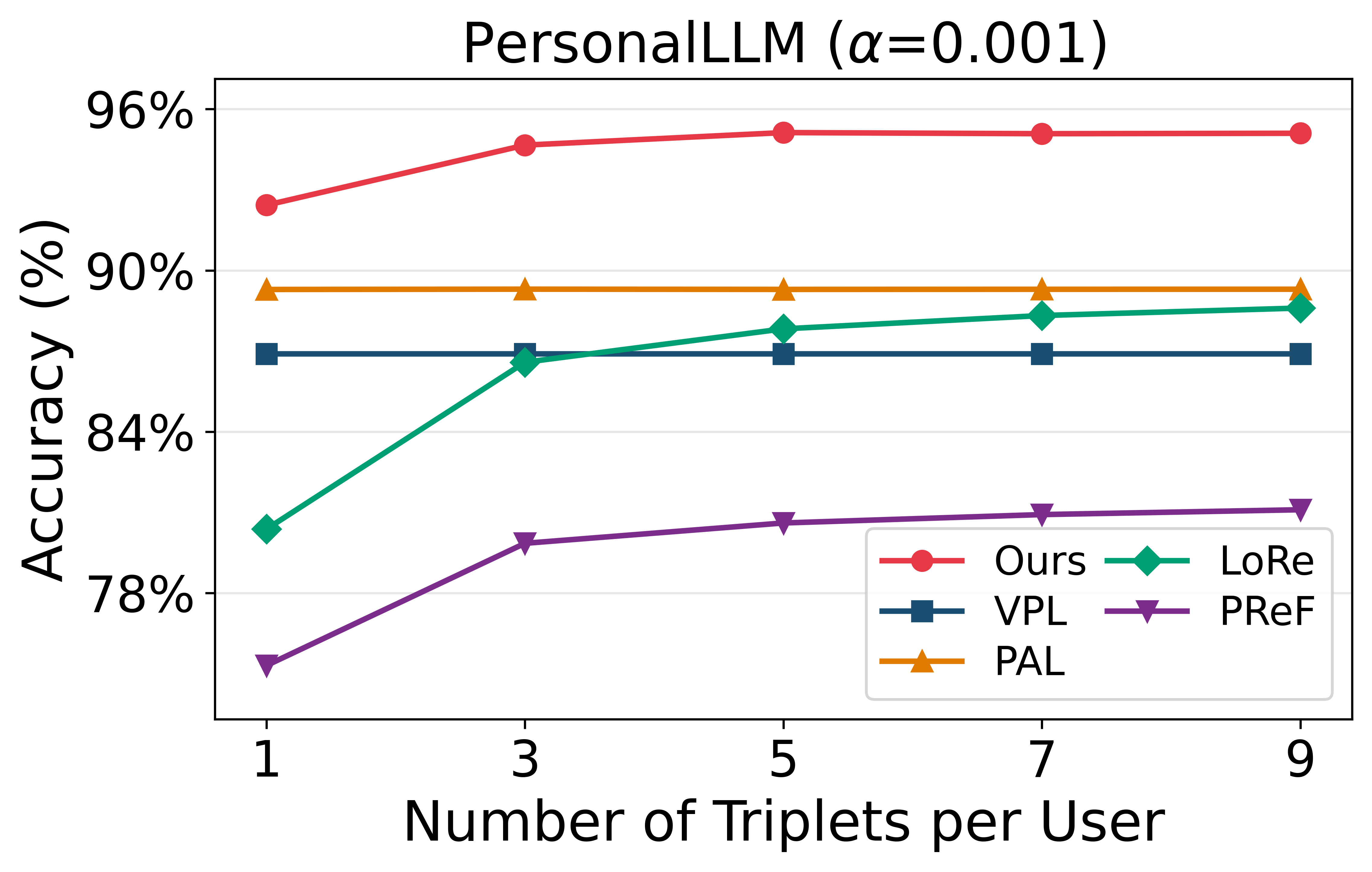}
        \hfill
        \includegraphics[width=0.48\linewidth]{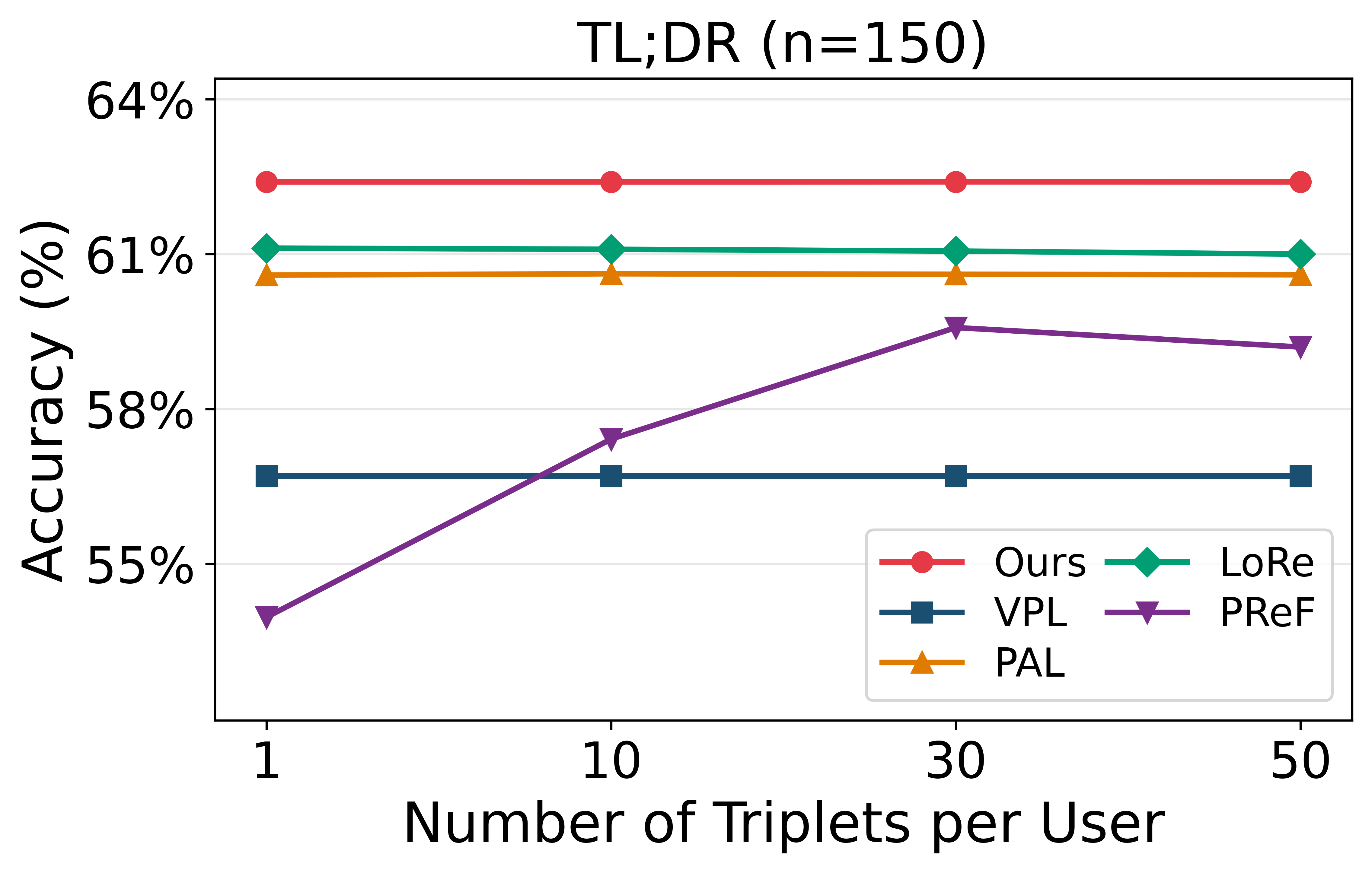}
        \subcaption{Few-shot adaptation (Varying $|\mathcal{C}_u|$)}
        \label{fig:fewshot_2}
    \end{minipage}

    \vspace{5pt}
    
    \begin{minipage}[b]{\linewidth}
        \centering
        \includegraphics[width=0.48\linewidth]{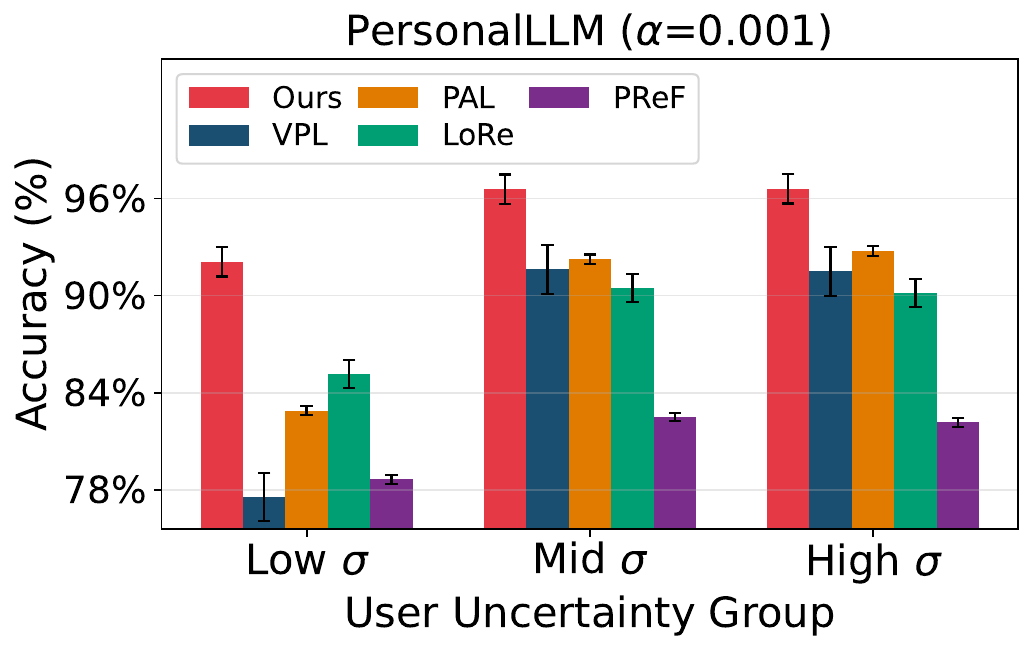}
        \hfill
        \includegraphics[width=0.48\linewidth]{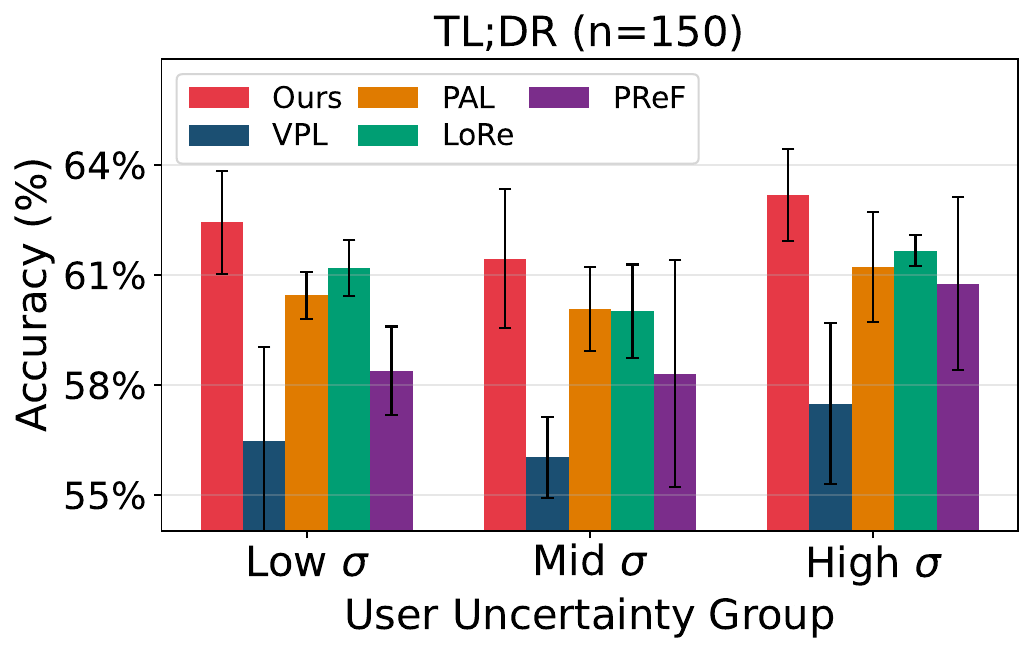}
        \subcaption{Uncertainty robustness}
        \label{fig:uncert_2}
    \end{minipage}

    \vspace{5pt}

    \begin{minipage}[b]{\linewidth}
        \centering
        \includegraphics[width=0.48\linewidth]{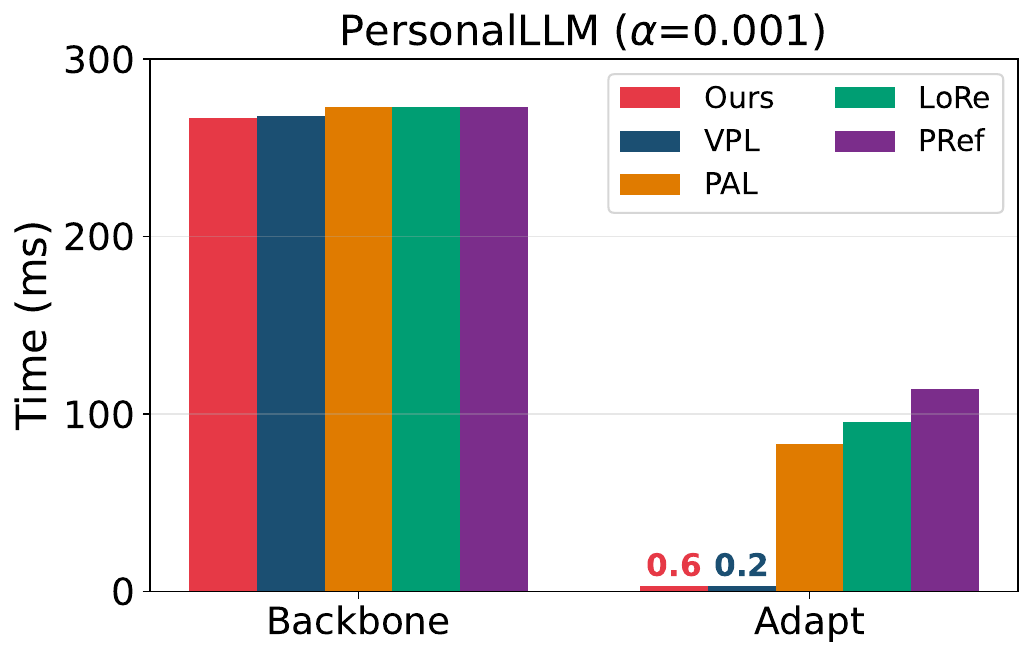}
        \hfill
        \includegraphics[width=0.48\linewidth]{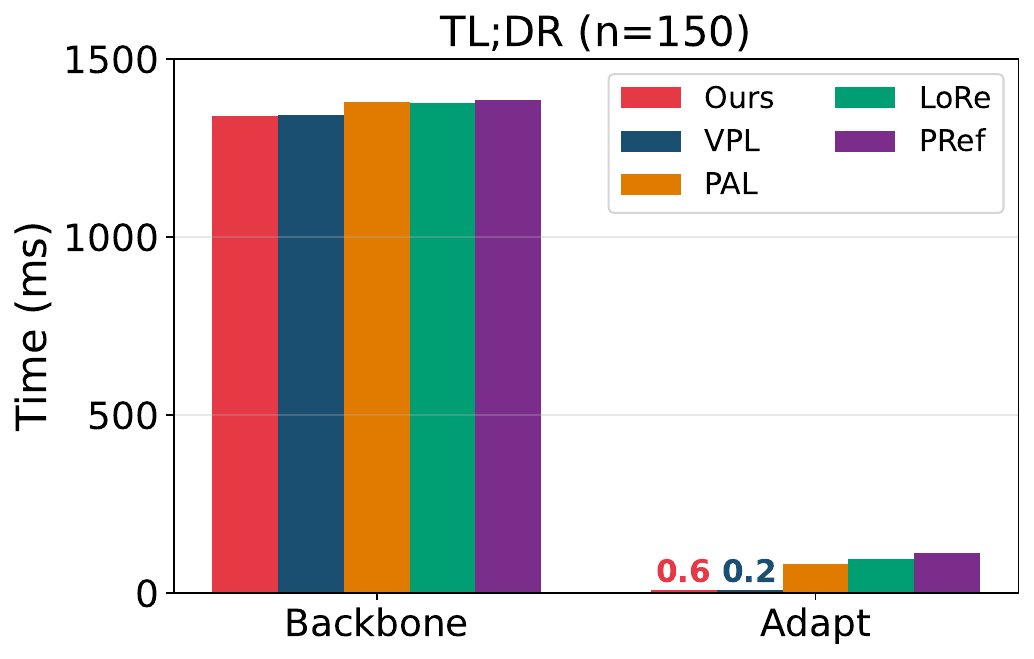}
        \subcaption{Per-user adaptation time (ms)}
        \label{fig:adapt_time_2}
    \end{minipage}
    
    \caption{In-depth comparison on PersonalLLM ($\alpha{=}0.001$) and TL;DR ($n{=}150$) unseen users.}
    \label{fig:combined_extra}
\end{figure*}

\vspace{100pt}
\subsection{Coverage and Generalization of Preference Bases}
\begin{wrapfigure}{r}{0.6\textwidth}
  \vspace{-20pt}
  \centering
  \includegraphics[width=0.55\textwidth]{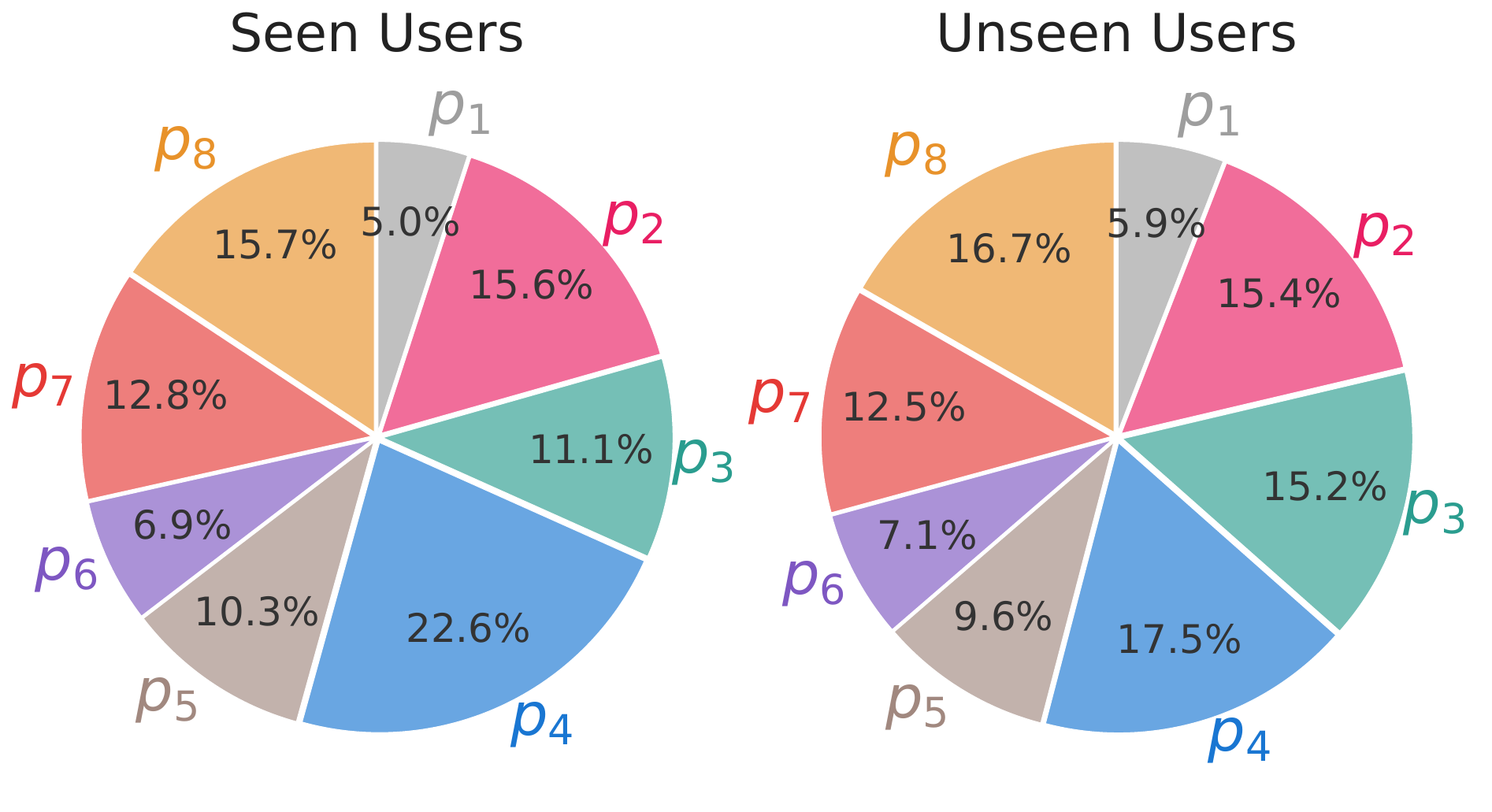}
  \caption{Distribution of dominant preference bases for seen and unseen users on PRISM.}
  \label{fig:pie_chart}
  \vspace{-10pt}
\end{wrapfigure}
Figure~\ref{fig:pie_chart} visualizes the distribution of dominant preference bases across seen and unseen users on PRISM, derived from the learned preference space in Figure~\ref{fig:pref_space}.
Two observations emerge: (1) users are spread across all K preference bases, suggesting that the bases effectively capture diverse user preferences.
(2) The two distributions exhibit similar trends, indicating that the learned preference bases generalize to unseen users.


\section{Proofs and Derivations}
\label{app:proof}
\subsection{Concavity of $\sigma^2_\Delta$ with Respect to $\mathbf{w}_u$}
\label{app:concavity}

\begin{proposition}
$\sigma_\Delta^2 = \sum_{k=1}^K w_{u,k}(\Delta\phi_k - \mu_\Delta)^2$ is concave with respect to $\mathbf{w}_u$.
\end{proposition}

\begin{proof}
The Hessian of $\sigma^2_\Delta$ with respect to $\mathbf{w}_u$ is:
\begin{equation}
    \nabla^2_{\mathbf{w}_u} \sigma^2_\Delta = -2 \Delta\boldsymbol{\phi} \Delta\boldsymbol{\phi}^\top.
\end{equation}
For any vector $\mathbf{v} \in \mathbb{R}^K$,
\begin{equation}
    \mathbf{v}^\top \nabla^2_{\mathbf{w}_u} \sigma^2_\Delta \, \mathbf{v} = -2 \mathbf{v}^\top (\Delta\boldsymbol{\phi} \Delta\boldsymbol{\phi}^\top) \mathbf{v} = -2 (\mathbf{v}^\top \Delta\boldsymbol{\phi})^2 \leq 0,
\end{equation}
so the Hessian is negative semi-definite and $\sigma^2_\Delta$ is concave. 
At the vertices of the simplex (one-hot $\mathbf{w}_u$), $\sigma^2_\Delta$ reduces to zero and increases as $\mathbf{w}_u$ becomes more diffuse.
\end{proof}

\subsection{Proof of Gradient Attenuation in $\mathcal{L}_{\mathrm{VBT}}$}
\label{app:proof_vbt}

\begin{proposition}
As $\sigma_\Delta^2 \to \infty$, the gradients of $\mathcal{L}_{\mathrm{VBT}}$ with respect to all parameters $\phi_k$ and $w_{u,k}$ vanish, regardless of the sign of $\mu_\Delta$.
\end{proposition}

\begin{proof}
Let $s = \sigma_\Delta^2$, $c = \frac{\pi}{8}$, and $z = \frac{\mu_\Delta}{\sqrt{1+cs}}$. We first compute the partial derivatives of $\mathcal{L}_{\mathrm{VBT}} = -\log\sigma(z)$ with respect to the intermediate variables $\mu_\Delta$ and $s$:

\begin{align}
\frac{\partial \mathcal{L}_{\mathrm{VBT}}}{\partial \mu_\Delta} &= -(1-\sigma(z)) \cdot \frac{1}{\sqrt{1+cs}}, \\
\frac{\partial \mathcal{L}_{\mathrm{VBT}}}{\partial s} &= \frac{\mu_\Delta c (1-\sigma(z))}{2(1+cs)^{3/2}}.
\end{align}

As $s \to \infty$,  $\frac{1}{\sqrt{1+cs}} \to 0$ and $\frac{1}{(1+cs)^{3/2}} \to 0$. 
Since $(1-\sigma(z))$ is bounded by $[0, 1]$, both $\frac{\partial L_{VBT}}{\partial \mu_\Delta}$ and $\frac{\partial L_{VBT}}{\partial s}$ converge to zero.

\textbf{Gradient with respect to $\phi_k$.} By chain rule:
\begin{equation}\label{eq:gradient_phi}
\frac{\partial \mathcal{L}_{\mathrm{VBT}}}{\partial \phi_k} = \frac{\partial \mathcal{L}_{\mathrm{VBT}}}{\partial \mu_\Delta} \cdot w_{u,k} + \frac{\partial \mathcal{L}_{\mathrm{VBT}}}{\partial s} \cdot \frac{\partial s}{\partial \phi_k}.
\end{equation}

\textbf{Gradient with respect to $w_{u,k}$.} By chain rule:
\begin{equation}\label{eq:gradient_w}
\frac{\partial \mathcal{L}_{\mathrm{VBT}}}{\partial w_{u,k}} = \frac{\partial \mathcal{L}_{\mathrm{VBT}}}{\partial \mu_\Delta} \cdot \Delta\phi_k + \frac{\partial \mathcal{L}_{\mathrm{VBT}}}{\partial s} \cdot (\Delta\phi_k - \mu_\Delta)^2.
\end{equation}
%
Since all terms in Eqs.~\eqref{eq:gradient_phi} and ~\eqref{eq:gradient_w} involve the vanishing partial derivatives $\frac{\partial L_{VBT}}{\partial \mu_\Delta}$ and $\frac{\partial L_{VBT}}{\partial s}$ multiplied by bounded coefficients, it follows that $\frac{\partial L_{VBT}}{\partial \phi_k} \to 0$ and $\frac{\partial L_{VBT}}{\partial w_{u,k}} \to 0$ as~$s \to \infty$.
\end{proof}


\subsection{Monte Carlo Estimation of $\mathcal{L}_{\mathrm{REG}}$}
\label{app:mc}

Since $D_{KL}(q_u \| p(z))$ against a MoG prior is intractable in closed form, we decompose it as:

\begin{equation}
    D_{KL}(q_u \| p(z)) = -\mathcal{H}(q_u) - \mathbb{E}_{q_u}[\log p(z)],
\end{equation}

where $\mathcal{H}(q_u)$ is the entropy of $q_u$. Each term is then estimated separately:

\textbf{Entropy term.} The entropy of a diagonal Gaussian $q_u = \mathcal{N}(\boldsymbol{\mu}_u, \mathrm{diag}(\boldsymbol{\sigma}_u^2))$ is given by:
\begin{equation}
    \mathcal{H}(q_u) = \frac{1}{2}\sum_{d=1}^D \left(1 + \log(2\pi) + \log \sigma_{u,d}^2\right).
\end{equation}

\textbf{Cross-entropy term.} $\mathbb{E}_{q_u}[\log p(z)]$ is estimated via the reparameterization trick~\citep{DBLP:journals/corr/KingmaW13} with $S{=}5$ Monte Carlo samples:
\begin{equation}
    \mathbb{E}_{q_u}[\log p(z)] \approx \frac{1}{S}\sum_{s=1}^S \log p(z^{(s)}), \quad z^{(s)} = \boldsymbol{\mu}_u + \boldsymbol{\sigma}_u \odot \boldsymbol{\epsilon}^{(s)}, \quad \boldsymbol{\epsilon}^{(s)} \sim \mathcal{N}(\mathbf{0}, \mathbf{I}),
\end{equation}
where $z^{(s)}$ is a reparameterized sample that allows gradients to flow back to $\boldsymbol{\mu}_u$ and $\boldsymbol{\sigma}_u$.
The MoG log-probability is computed as:
\begin{equation}
    \log p(z) = \log \frac{1}{K}\sum_{k=1}^K \mathcal{N}(z; \boldsymbol{\mu}_k, \mathrm{diag}(\boldsymbol{\sigma}_k^2)) = \mathrm{logsumexp}_k\!\left(\log \mathcal{N}(z; \boldsymbol{\mu}_k, \boldsymbol{\sigma}_k^2)\right) - \log K,
\end{equation}
with logsumexp applied for numerical stability. 
Thus, combining both terms yields a differentiable estimate of ${D}_{KL}(q_u \parallel p(z))$ that can be optimized end-to-end via backpropagation.

\section{Prompts}
\label{app:prompts}
\subsection{Preference Space Analysis}
\label{app:prompt_preference}
To interpret the learned preference space in Figure~\ref{fig:pref_space}, we label each preference basis and classify user triplets using Claude Sonnet 4.6.

\paragraph{Preference Basis Labeling.}
For each preference basis, we sample 10 triplets from the top-5 nearest users and use the following prompt:

\begin{tcolorbox}[colback=gray!5, colframe=black, fontupper=\small]
The following are 10 triplets from users with similar preferences.\\
Describe what drives their preference for the chosen over the rejected responses.\\[10pt]
{[1]} $(x, y^+, y^-)$\\[4pt]
\dots \\[4pt]
{[10]} $(x, y^+, y^-)$
\end{tcolorbox}
\noindent The resulting labels (e.g., \textit{conversational \& natural}, \textit{direct \& structured}) are reported in Figure~\ref{fig:pref_space}.

\paragraph{Triplet Classification.}
Using the $K$ labels, each triplet is classified by prompting:
\begin{tcolorbox}[colback=gray!5, colframe=black, fontupper=\small]
Given a preference triplet $(x, y^+, y^-)$, which of the following preference basis labels best explains the user's choice?\\[4pt]
{[1]} Conversational \& Natural\\
{[2]} Thorough \& Genuine\\
\dots\\
{[8]} Empathetic \& Principled\\[4pt]
Answer with only the number.
\end{tcolorbox}




\subsection{Inference-Time Alignment}
\label{app:prompt_inference_time_alignment}
In Figure~\ref{fig:winrate}, we evaluate whether the learned personalized rewards can guide response selection at inference time via best-of-N sampling.

\paragraph{Response Generation.}
We generate $N{=}8$ candidate responses using GPT-5 mini as a policy model, each conditioned on a distinct preference basis label as a response style.
\begin{tcolorbox}[colback=gray!5, colframe=black, fontupper=\small]
You are a helpful AI assistant. Generate a \{\textit{style}\} response to the following query: \{\textit{query}\}
\end{tcolorbox}

\paragraph{LLM-as-a-Judge.}
We employ GPT-5.1 as an evaluator, comparing each method's selected response against BT's, with the user's profile (e.g., demographics, LLM preferences, personal values) to guide evaluation. 
Response order is randomized to mitigate position bias.
\begin{tcolorbox}[colback=gray!5, colframe=black, fontupper=\small]
You are an impartial judge. Given the user's profile, the query, and the two responses, which is more appropriate for this user?\\[4pt]
{[User Profile]} \textit{\{user\_profile\}}\\[4pt]
{[Query]} \textit{\{query\}}\\[4pt]
{[Response A]} \textit{\{response\_a\}}\\[4pt]
{[Response B]} \textit{\{response\_b\}}\\[4pt]
Reply with ONLY ``A'' or ``B''.
\end{tcolorbox}

\end{document}